\documentclass[10pt,onecolumn,letterpaper]{article}

\usepackage{cvpr}

\usepackage{amsmath}
\usepackage{amssymb}
\usepackage{amsthm}

\usepackage{comment}
\usepackage{enumitem}
\usepackage{graphicx}
\usepackage{times}
\usepackage{float} 

\usepackage{amsmath,amsfonts,bm}

\def\eqref#1{equation~\ref{#1}}

\def\1{\bm{1}}

\DeclareMathAlphabet{\mathsfit}{\encodingdefault}{\sfdefault}{m}{sl}
\SetMathAlphabet{\mathsfit}{bold}{\encodingdefault}{\sfdefault}{bx}{n}

\newcommand{\benum}{\begin{enumerate}}
\newcommand{\eenum}{\end{enumerate}}
\newcommand{\nc}{\newcommand}
\newcommand{\rnc}{\renewcommand}

\nc{\<}{\langle}
\rnc{\>}{\rangle}
\nc{\bq}{\textbf}
\nc{\m}{\textrm}
\nc{\bb}{\mathbb}
\nc{\til}{\texttildelow}
\nc{\dps}{\displaystyle}
\rnc{\l}{\left(}\rnc{\r}{\right)}
\nc{\lc}{\left\{}\nc{\rc}{\right\}}
\nc{\lb}{\left[}\nc{\rb}{\right]}
\nc{\ba}[1]{\begin{array}{#1}}
\nc{\ea}{\end{array}}       
\nc{\ra}{\rightarrow}
\nc{\li}{\left |}
\nc{\ri}{\right |}
\nc{\pde}[2]{\frac{\partial #1}{\partial #2}}
\nc{\ode}[2]{\frac{d #1}{d #2}}
\nc{\odee}[3]{\frac{d^{#3} #1}{d #2^{#3}}}
\nc{\pdee}[3]{\frac{\partial^{#3} #1}{\partial #2^{#3}}}
\nc{\bn}{\begin{enumerate}}
\nc{\en}{\end{enumerate}}
\nc{\bt}{\begin{theorem}}
\nc{\et}{\end{theorem}}
\nc{\y}[1]{\lambda_{#1}}
\nc{\ninf}{{\oplus}^{-\infty}}
\nc{\pinf}{{\oplus}^{+\infty}}
\nc{\nninf}{{\otimes}^{-\infty}}
\nc{\ppinf}{{\otimes}^{+\infty}}
\nc{\ir}{\mathbb{I}\mathbb{R}}
\nc{\closure}[2][3]{{}\mkern#1mu\overline{\mkern-#1mu#2}} 
\nc{\ep}{\mathcal{E}_{P}}
\nc{\mr}{\mathcal{M}_{r}}
\nc{\mfa}{\mathcal{M}_{f,a}}
\nc{\mfp}{\mathcal{M}_{f,p}}
\nc{\mt}{\m{T}}
\nc{\dB}{\textit{dB}}
\nc{\VM}{\delta} 
\nc{\GM}{\Theta}
\nc{\PM}{\phi} 
\nc{\x}{\mathbf{x}}

\usepackage{xcolor}

\newcommand{\R}{\mathbb{R}}

\newtheorem{theorem}{Theorem}

\newtheorem{lemma}[theorem]{Lemma}
\newtheorem{remark}[theorem]{Remark}

\newtheorem{defn}{Definition}

\usepackage[breaklinks=true,bookmarks=false]{hyperref}

\cvprfinalcopy 

\def\cvprPaperID{8641} 

\usepackage{cleveref}
\begin{document}

\title{Translation Insensitive CNNs}

\author{Ganesh Sundaramoorthi\\
United Technologies  Research Center \\
411 Silver Ln. \\
Hartford, CT 06118 \\
{\tt\small sundarga1@utrc.utc.com }
\and
Timothy E. Wang \\
United Technologies  Research Center \\
411 Silver Ln. \\
Hartford, CT 06118 \\
{\tt\small wangte@utrc.utc.com}
}

\maketitle
\begin{abstract}
We address the problem that state-of-the-art Convolution Neural Networks (CNN) classifiers are not invariant to small shifts. The problem can be solved by the removal of sub-sampling operations such as stride and max pooling, but at a cost of severely degraded training and test efficiency.  We present a novel usage of Gaussian-Hermite basis to efficiently approximate arbitrary filters within the CNN framework to obtain translation invariance. This is shown to be invariant to small shifts, and preserves the efficiency of training. Further, to improve efficiency in memory usage as well as computational speed, we show that it is still possible to sub-sample with this approach and retain a weaker form of invariance that we call \emph{translation insensitivity}, which leads to stability with respect to shifts. We prove these claims analytically and empirically. Our analytic methods further provide a framework for understanding any architecture in terms of translation insensitivity, and provide guiding principles for design.


\end{abstract}

\section{Introduction}
\label{sec:intro} 
Deep neural nets have revolutionized computer vision. Convolutional neural networks (CNNs) are particularly popular in computer vision, especially for tasks such as classification, object detection, and segmentation. CNNs are motivated by certain invariances that are claimed to be inherent in their architectures. For instance, one of these claimed invariances is invariance to in-plane translations in the image, i.e., the output of a CNN is left unchanged as the object is translated, a desirable property for an object classifier. That property is claimed to be built into the architecture, so that such invariance is exhibited regardless of the parameters of the CNN that are learned. The choice of convolution and pooling in CNN layers are said to result in such translation invariance. For more general properties of CNNs related to invariances, see~\cite{soatto2014visual,anselmi2016unsupervised,achille2018emergence} and~\cite{lenc2015understanding} for experimental analysis.

However, recent work~\cite{azulay2018deep} has shown that existing CNNs (e.g., especially ResNets~\cite{he2016deep}, VGG~\cite{simonyan2014very}, Inception Net~\cite{szegedy2017inception}) are highly unstable to basic transformations including small translations in the image, scalings and natural perturbations across frames in videos (see also,~\cite{shankar2019systematic}) that one would expect CNNs to be invariant (or stable to). In fact this is not an odd occurrence, as in adversarial perturbations, that are specifically constructed to ``attack" the network, but rather it is a common occurrence. For instance, it is reported in~\cite{azulay2018deep} that even a 1 pixel shift in an input test image can cause a change in the resulting classification of the CNN with probability 30\% in existing modern CNNs. This is despite typical data augmentation done in training, i.e., augmenting the training set with shifted versions of the training images.

It is claimed by~\cite{azulay2018deep} that this lack of translation invariance is caused by the sub-sampling or the pooling operation in CNNs between layers, which is said to ignore the classical Shannon-Nyquist sampling theorem. This is also noted as early as~\cite{simoncelli1992shiftable}. The claim is that one would be limited to sub-sample at a rate compatible with the Nyquist rate, i.e., twice the highest frequency in the feature map before the sub-sampling. Thus, to avoid aliasing effects, one would need to blur the feature map to eliminate high frequencies to support the desired sub-sampling rate. That is, however, not performed in existing CNNs.  However, as claimed in~\cite{azulay2018deep}, the approach does not achieve full translation invariance. It is claimed that the reason is due to the presence of a non-linearity, which may introduce aliasing even in the presence of blur.

The only solution, as suggested by~\cite{azulay2018deep}, seems to be to avoid sub-sampling. However, that leads to two problems: 1) to retain the receptive field sizes of conventional CNNs with sub-sampling, the kernel supports in layers would have to grow with respect to the depth of the layer, but this would require learning an exorbitant number of parameters to represent the growing kernel support, which poses a problem for training efficiency
(see~\cite{yu2015multi} for a possible solution), and 2) the memory footprint due to non-sub-sampled feature maps is a problem in back-propagation during the training process, which may preclude the use of many GPUs that do not have sufficient memory, and the lack of sub-sampling also leads to greater training and inference times.

In this paper, we tackle both of these aforementioned problems. To do this, we represent convolutional kernels with an orthogonal Gauss-Hermite basis whose basis coefficients are learned in convolution layers, rather than representing a kernel directly in terms of its pixel values\footnote{This basis has been used before~\cite{jacobsen2016structured} in CNNs, but the essential properties of translation invariance and insensitivity are not mentioned or explored in that work, and the focus is rather on the reduced parameters.}. Without a sub-sampling layer, the representation leads to a fully translation invariant representation that keeps constant the number of parameters in kernels across layers, while being able to capture as big receptive fields as modern CNNs, with even fewer parameters per layer. To address the memory limitations, we show how the layers, due to the smooth Gauss-Hermite approximation, can be sub-sampled, in a way that retains a weaker notion of translation invariance that we term \emph{translation insensitivity}. This leads to stability of classifications with respect to translations, in contrast to existing sub-sampled CNNs, which do not exhibit translation insensitivity.

{\bf Contributions:} Our specific contributions are as follows. {\bf 1.} We introduce a CNN architecture, which we call \emph{GaussNets}, that is both translation invariant and has equivalent (or bigger) receptive fields with fewer parameters per kernel than existing modern CNNs. {\bf 2.} We introduce a CNN architecture, called \emph{Sub-sampled GaussNets}, that exhibit the aforementioned properties with respect to receptive field and parameter usage, and is computationally efficient comparable to modern CNNs, by performing sub-sampling.  This architecture retains a weaker form of translation invariance, which we call \emph{translation insensitivity} that gives robustness to classifications. {\bf 3.} We provide analytic proofs that show that these introduced architectures exhibit the aforementioned properties. These analytic tools serve as a framework for analyzing any other architecture. {\bf 4.} We experimentally demonstrate the insensitivity to translation.

\subsection{Related Work}




The lack of translation invariance in modern CNNs (such as ResNet, VGG, InceptionNet) is due to the sub-sampling or pooling operations. One obvious approach to deal with this lack of invariance is by data augmentation - augmenting the training set with shifted test images. However, 
\cite{azulay2018deep} shows that that only gives invariance to shifts on images statistically very similar to the training set, and the lack of invariance remains on the test set. The only work, to the best of our knowledge, that attempts to address the lack of translation invariance, due to sub-sampling is~\cite{zhang2019making}. It is proposed to add an anti-aliasing layer by applying a fixed smoothing filter before the sub-sampling. Though the approach is not translation invariant, it is shown to empirically improve robustness to translation on the test set. We show analytically that if the right smoothing kernel is chosen, such anti-aliasing would give what we call translation insensitivity, something that was not known. While anti-aliasing does provide a solution, the approach we introduce, in addition lends itself to other invariances such as scaling and deformation that we will be the subject of our future work.


In~\cite{mallat2012group,bruna2013invariant,sifre2013rotation} (see also~\cite{malik1990preattentive} for an early approach), the \emph{Scattering Transform} is introduced as a representation of an image invariant to basic transformations, such as translation, scale, rotations, and small deformations for classification. The transform is computed by convolving the image with a wavelet filter bank, a pointwise non-linearity (complex modulus), followed by a low-pass filtering. These operations are then stacked to create a hierarchical representation. It is proven that such a representation provides not strict invariance, but a weaker notion, that we refer to as \emph{insensitivity}. We note that all such proofs are done assuming continuous data, and the sub-sampling operations are not analyzed. The transform has been used in a number of classification problems such as texture discrimination with success, but since the 
features are hand-crafted, it seems difficult to apply them to complex classification tasks such as ImageNet, where hand-crafting feature combinations would be difficult to compete with learned CNN approaches. Some more recent approaches along the lines of~\cite{mallat2012group,bruna2013invariant,sifre2013rotation} are~\cite{cohen2016steerable,cheng2018rotdcf} (based on steerable filters~\cite{freeman1991design,perona1992steerable}) that obtain equivariance, i.e., a group action such as rotation on the input results in the same rotation in the feature map, while having learned parameters like CNNs.


In~\cite{jacobsen2016structured}, a hybrid approach between Scattering transforms and CNNs is taken, motivated by the desire to perform well on both small datasets with limited training data and large datasets. In that work, it is observed that the filters learned in CNNs from large datasets often are spatially coherent, and thus, rather than learning that spatially coherent structure directly from data, it is enforced in the filters to limit the training data requirements. To that end, a fixed structured basis, similar to the filter bank in Scattering transforms, is chosen to represent the kernel, and the coefficients are learned like CNNs. The particular basis chosen is derivatives of Gaussians~\cite{koenderink1987representation}, which is coincidentally the same basis that we use in this paper. However, our motivation is quite different than~\cite{jacobsen2016structured}; indeed, we seek to obtain translation invariance in CNNs, which we show is maintained by our architecture without sub-sampling. Furthermore, we show that translation insensitivity is maintained despite sub-sampling. Such properties were not explored in~\cite{jacobsen2016structured}. Another approach using Gaussian filters, though not a basis, include~\cite{shelhamer2019blurring}, where the filter shape and size are learned.~\cite{tabernik2019spatially} uses a mixture of Gaussians in different spatial locations in defining kernels, which is parameter efficient.




\section{GaussNet CNN Architecture} 
\label{sec:gaussnet_arch}
We introduce our GaussNet architecture. The key idea is that rather than representing a convolution kernel directly in terms of its pixelized values (corresponding to coefficients of a basis of shifted delta functions), we represent the kernel in terms of an (orthogonal) basis of smooth functions given by derivatives of the Gaussian. In fact, any $\mathbb{L}^2(\mathbb{R}^2)$ function can be approximated as
\begin{equation}
	K(x) = \sum_{i=0}^{\infty} a_i \odot D^i G_{\sigma}(x),
\end{equation}
where $G_{\sigma}$ is the 2d Gaussian  function with standard deviation $\sigma$, $D^i$ represents the $i^{\text{th}}$ derivative operator (the tensor of all partials of up to $i$ derivatives), $a_i$ is a tensor of coefficients that are $\ell^2(\mathbb{N})$, and $\odot$ represents the sum of element-wise products between the two arguments.  In practice, we will use an approximation of up to order 2 derivatives, so that the kernels we represent are given by
\begin{equation} \label{eq:gauss_kernel_approx}
	K(x) = W DG_{\sigma}(x),
\end{equation}
where $W\in \R^{1\times 6}$ represent the coefficients of the Gaussian derivatives up to order $2$, and $D = (1,\partial_x, \partial_y, \partial_{xx}, \partial_{xy}, \partial_{yy})^T$ represents a vector of partials in the $x$ and $y$ directions. The representation in \eqref{eq:gauss_kernel_approx} will be used rather than pixelized representations typically used (e.g., $3\times 3$ pixelized kernels used in VGG and ResNet). The weights $W$ will be learned. By using this basis, the kernels are enforced to be smooth, while being flexible enough to have the discrimination power to separate object classes. As we will show in the next section, the smoothness of this choice of basis, leads to the translation insensitivity that we desire.

A basic layer of a GaussNet is be given by 
\begin{equation} \label{eq:gaussnet_basic_layer}
	f_W[I] = S_d \left\{ r( W DG_{\sigma}\ast I ) \right\}
\end{equation}
where $W\in \R^{m\times 6n}$ is a weight matrix and $I$ is an $n$-dimensional input image (feature), $m$ is the output feature map dimension, and $\ast$ is the convolution. Thus, each input channel is convolved with derivatives of Gaussians, and linear combinations of these are formed with the weight matrix $W$.   $r$ is the rectified linear unit, i.e., $r(x)=\max\{x,0\}$. $S_d$ represents the sub-sampling operator.  In one form of our architecture that is fully translation invariant, $S_d$ is not included. The sub-script in $f_W$ is used to indicate the free parameters that are to be learned. In experiments, we will demonstrate our approach on an architecture motivated from ResNet. In this architecture, one sums the input feature map to a layer with the result above after the rectification, i.e.,
\begin{equation}
	f_W[I] = S_d \left\{ I + r( W DG_{\sigma}\ast I ) \right\}.
\end{equation}
In the analysis in the next section, for simplicity, we will analyze the basic layer in \eqref{eq:gaussnet_basic_layer}, but the results we prove will also hold for a ResNet-like layer.

Multiple layers will be cascaded to form a deep CNN, which we call the GaussNet. As in a ResNet-like structure, our final feature will consist of an average pooling layer, i.e.,
\begin{equation}
	F[I](x) = \frac{1}{N} \sum_x f_{W_{N_{l}}} \circ \cdots \circ f_{W_1}[I](x),
\end{equation}
where $N$ is the number of pixels in the final feature maps, and $N_l$ is the number of layers.

Note that the effective receptive field size of the GaussNet is controlled by the parameter $\sigma$. In the GaussNet without sub-sampling, in order to maintain the overall receptive field size of the corresponding traditional CNN, the parameter $\sigma$ would have to grow according to the sub-sampling rate of the traditional CNN, i.e., $\sigma d^{\ell-1}$ where $\ell$ is the layer number. We efficiently evaluate such large receptive field convolutions with the Fast Fourier Transform (FFT) (see Section~\ref{sec:expts}).  In the sub-sampled GaussNet, $\sigma$ will remain fixed over layers, and because of the sub-sampling, the overall receptive field would be similar to the corresponding traditional CNN.

Note that in comparison to existing CNNs, the GaussNet has fewer parameters, i.e., a common choice (e.g., ResNet and VGG) is to use $3\times 3$ in convolution operations - this results in 9 parameters that should be learned, whereas we use 6 coefficients per convolution filter, while still being able to obtain a similar test accuracy as the traditional CNN.

\section{Translation Insensitivity of GaussNets}
\label{sec:new_arch}

\subsection{Terminology}
In this section, we will denote an image or feature map as $I$, where $I : [0,m_1-1]\times [0,m_2-1] \to \R$. In our proofs, we will assume that the data is defined on the infinite discrete set $\mathbb{Z}$ for simplicity of notation, as we may just zero-pad the finite data with an infinite number of zeros. The set of all such images is denoted $\mathcal{I}_{m_1\times m_2}$. For simplicity in the notation, we will consider just one feature map in the input and output of each layer. We will denote an operation from one feature map to produce another as $f : \mathcal{I}_{m_1\times m_2} \to \mathcal{I}_{o_1\times o_2}$. This will typically correspond to an output of a layer of the network. We will denote an operation from an image or feature map to a vector as $F : \mathcal{I}_{m_1\times m_2} \to \R^{M}$.  This will typically correspond to the last layer of the network that produces a vector representation of the image.

We define notation for some operations that we will be used extensively in the rest of the paper.
We denote $T_s$ to be the \emph{translation operator}, i.e., 
\begin{equation}
	T_sI(x) \triangleq I(x+s),
\end{equation}
which shifts the image by $s$. Note that the translation operator is defined only on infinite domains, as $x+s$ must always be in the domain of the image. However, for finite data, this can be extended to an infinite domain by zero padding the finite data. Next, we define the \emph{sub-sampling operator} $S_d$, which sub-samples data by a factor of $d>1$, as follows
\begin{equation}
	S_dI(x) \triangleq I(xd).
\end{equation}

In the following sections, we will show that the behavior of GaussNets are well-behaved with respect to the translation operator. We will now make precise this behavior.

\begin{defn}[$g$-Translation Covariant Operator]
	An operator $f$ is \emph{$g$-translation covariant} if for all translations $T_s$ and all inputs $I$, we have that
	\begin{equation} 
		f[T_s I] = T_{g(s)}f[I],
	\end{equation}
	where $g : \R\to\R$ is a monotone bijective function. When $g(s)=s$, we simply say that $f$ is \emph{translation covariant}.
\end{defn}
\noindent This definition says that a shift in the input map of a $g$-translation covariant operator results in a predictable shift in the output map.

We now introduce the notion of translation invariance; ideally, a property inherent in a CNN.
\begin{defn}[Translation Invariance]
	A function $F$ is \emph{translation invariant} if for all translations $T_s$ and all inputs $I$, we have that
	\begin{equation}
		F[ T_s I ] = F[I].
	\end{equation}
\end{defn}
This says that the feature does not change as the image is translated.
In practice, we will have to settle for a weaker property, which we call \emph{translation insensitivity}:
\begin{defn}[Translation Insensitivity]
	A function $F$ is \emph{translation insensitive} if there exists a positive constant $C < \infty$ such that 
	\begin{equation}
		|F[T_s I] - F[I]| \leq C |s|,
	\end{equation}
	for all $s$.
\end{defn}
This is what is known as Lipschitz continuity: the feature representation does not change much, i.e., at most at a linear rate of the shift size.

\subsection{Invariance of GaussNets}

We define a layer of a GaussNet (without sub-sampling) as follows:
\begin{equation} \label{eq:layer_GaussNet}
	f[I] = r ( [WDG_{\sigma}]\ast I ),
\end{equation}
where $r(x)=\max\{x,0\}$ is the rectified linear unit, $G_{\sigma}$ is the Gaussian, $D = (1, \partial_x, \partial_y, \partial_{xx}, \partial_{xy}, \partial_{yy})^T$, $W\in \R^{1\times 6}$ is a weight matrix, and $\ast$ is the convolution operator.

We show now a network consisting of stacking layers defined in \eqref{eq:layer_GaussNet} is translation covariant. In fact, the property is true for any CNN that performs no sub-sampling (or down-sampling), and thus also a GaussNet:
\begin{lemma}[GaussNets are Translation Covariant]
	A deep GaussNet of the form
	\begin{equation}
		\tilde{f} = f_n \circ f_{n-1} \circ \cdots \circ f_1,
	\end{equation}
	where $f_i$ is in the form \eqref{eq:layer_GaussNet}, which does no sub-sampling is translation covariant.
\end{lemma}
\begin{proof}
	The composition of translation covariant operators is also translation covariant. The convolution is translation covariant, as is the rectified linear unit as it is a point-wise function of the input, therefore $f_i$ is translation covariant, and thus so is the composition $\tilde f$.
\end{proof}

The translation covariant property of a composition of several GaussNet layers allows us to now define a network from this composition that is translation invariant by using average pooling. Note that this also holds for ordinary CNN layers that perform no sub-sampling / down-sampling.
\begin{theorem}[Average Pooling of a GaussNet is Translation Invariant]
	The deep GaussNet followed by an average pooling layer at the end, i.e.,
	\begin{equation}
	 	F[I] = \frac 1 N \sum_x f_n \circ f_{n-1} \circ \cdots \circ f_1[I](x)
	\end{equation}
	is translation invariant.
\end{theorem}
\begin{proof}
	Using the translation covariant property of the composition of Gaussian layers, we have that
	\begin{align}
		F[T_s I] &= \frac 1 N \sum_x T_s[ f_n \circ f_{n-1} \circ \cdots \circ f_1[I] ](x)\\
		&=\frac 1 N \sum_x f_n \circ f_{n-1} \circ \cdots \circ f_1[I](x+s) \\
		&=\frac 1 N \sum_z f_n \circ f_{n-1} \circ \cdots \circ f_1[I](z) \\
		&= F[I],
	\end{align}
	where $z = x+s$ and a change of variables was performed.
\end{proof}

\subsection{Insensitivity of Sub-sampled GaussNets}

While the GaussNets in the previous section exhibited translation invariance, and are able to capture as large receptive fields as conventional CNNs with down-sampling (with appropriate choice of $\sigma$ in each layer), with fewer parameters per layer, the lack of down-sampling poses a problem for back-propagation as large memory requirements, preclude some GPUs. Thus, we would like to sub-sample in practice.  However, strict translation invariance is lost; however, we show that translation insensitivity is retained.

We define a layer of the sub-sampled GaussNet as:
\begin{equation} \label{eq:subsample_network_layer}
	f[I] = S_d[r( [WDG_{\sigma}] \ast I )],
\end{equation}
where, as defined before, $S_d$ is the sub-sampling operator. We first show that a single layer GaussNet with sub-sampling followed by an average pooling is translation insensitive. We will then use this result to generalize to the multi-layer case.
\begin{theorem}\label{thrm:insense_sub_Gauss_layer}
	Average pooling of a layer of the sub-sampled GaussNet in \eqref{eq:subsample_network_layer} is translation insensitive, i.e.,
	\begin{equation}
		\left| \frac 1 N \sum_x f[T_s I](x) - \frac 1 N \sum_x f[I](x) \right| \leq C |s|,
	\end{equation}
	where $C$ is a constant and $N$ is the number of pixels in the feature map. The constant $C$ is given by 
	\begin{equation}
		C = C_G N \| I \|_{\infty} \| W \|_{\infty},
	\end{equation}
	where $C_G$ is the Lipschitz constant of the Gaussian and its derivatives.
\end{theorem}
\begin{proof}
	We compute 
	\begin{align}
		[WDG_{\sigma}] \ast I(x) &=  \sum_y \sum_j w_j D_j G_{\sigma}(x-y) I(y)
	\end{align}
	and
	\begin{align*}
		[WDG_{\sigma}] \ast [T_sI](x) &=  \sum_y \sum_j w_j D_j G_{\sigma}(x-y) I(y+s)\\
		&= \sum_y \sum_j w_j D_j G_{\sigma}(x-y-s) I(y),
	\end{align*}
	where we have performed a change of variables. 
	Thus,
	\begin{equation}
		[ (WDG_{\sigma}) \ast I](x) - [ (WDG_{\sigma}) \ast (T_sI) ](x) = 
		\sum_{y,j} w_j [D_j G_{\sigma}(x-y)-D_j G_{\sigma}(x-y-s)] I(y).
	\end{equation}
	Now,
	\begin{align}
		\left|   \frac 1 N \right.  & \left. \sum_x f[I](x) - \frac 1 N \sum_x f[T_sI](x)\right|   \\
		&\leq \frac 1 N \sum_x | f[I](x)-f[T_sI](x) | \\
		&\leq \frac 1 N \sum_x
		\left| 
			S_d\left\{ r[K \ast I](x) - r[ K \ast (T_sI) ](x)\right\}
		\right| \\
		&\leq \frac 1 N \sum_x
		\left| 
			r[K \ast I](xd) - r[K \ast (T_sI) ](xd)
		\right|	\\
		&\leq \frac 1 N \sum_x
		\left| 
			[K \ast I](xd) - [K \ast (T_sI) ](xd)
		\right|,
	\end{align}
	where $K=WDG_{\sigma}$ and the last inequality is due to the Lipschitz continuity of the rectified linear unit, i.e., $|r(x)-r(y)| \leq |x-y|$. Therefore,
	\begin{equation}
		\left|\frac 1 N \sum_x f[I](x)- \frac 1 N \sum_x f[T_sI](x)\right| \leq 
		\frac 1 N \sum_{x,y,j} |w_j| |[D_j G_{\sigma}(xd-y)-D_j G_{\sigma}(xd-y-s)]| |I(y)|.
	\end{equation}
	Since the Gaussian (and its derivatives) are Lipschitz continuous, with Lipschitz constant $C_G$, we have that
	\begin{align*}
		\left|\frac 1 N \sum_x f[I](x)- \frac 1 N \sum_x f[T_sI](x)\right| &\leq C_G N \| I \|_{\infty} \| W \|_{\infty} |s|.
	\end{align*}
\end{proof}

We now proceed to show that average pooling of a multi-layer GaussNet is also translation insensitive. To do this, we first show that a layer of a GaussNet with sub-sampling is $g$-translation covariant with $g(s)=s/d$:
\begin{lemma}[Sub-Sampled GaussNet Layer is $g$-Translation Covariant]
	The GaussNet with sub-sampling layer defined in \eqref{eq:subsample_network_layer} is translation covariant (a shift of $s$ in the input corresponds to a shift of $s/d$ in the output map).
\end{lemma}
\begin{remark}[Fractional Shifts]
	Note that $s/d$ may not be an integer. Since the image/feature map is defined discretely, we specify how this operation is defined. We will see that the answer will arise naturally in the proof: the fractional shifted feature map will be defined by a formula that needs only the Gaussian on the non-sub-sampled domain, where $s/d$ will correspond to an integer shift.
\end{remark}
\begin{proof}
From the proof of the previous theorem, we have that 
\begin{align}
	S_d & \left\{ \right. \left. r[ (WDG_{\sigma}) \ast (T_sI)] \right\}(x) \\
	&= r\left\{ \sum_y \sum_j w_j D_j G_{\sigma}(xd-y-s) I(y) \right\} \\
	&=  r\left\{\sum_y \sum_j w_j D_j G_{\sigma}[ (x-s/d)d-y ] I(y) \right\}\\
	&= S_d\left\{ r[ (WDG_{\sigma}) \ast I] \right\}(x-s/d).
\end{align}
Note in the previous equation, the spatial argument is discrete and $s/d$ may not be an integer. However, the expression makes perfect sense as it is applied as an argument to the Gaussian, which applies in the non-sub-sampled domain where the fractional shift corresponds to an integral value.  Therefore, by the translation covariance of rectified linear unit, we have that
\begin{equation}
	f[T_sI] = T_{s/d} f[I].
\end{equation}
\end{proof}

We now move to proving that average pooling for a deep GaussNet is translation insensitive.  First, we need a lemma that shows that if the input feature maps are close, then the outputs through a layer of a GaussNet are close.
\begin{lemma} \label{lem:insensitive_input}
	If $I$ and $J$ are feature maps such that $|J_1(x)-J_2(x)|\leq C_I|s|$ for all $x$ and $f$ is a GaussNet layer with sub-sampling, then 
	\begin{equation}
		\max_x | f[J_1](x)-f[J_2](x) | < C_f |s|,
	\end{equation}
	where $C_f = C_I N \| W\|_{\infty} \| DG_{\sigma}\|_{\infty}$, and $N$ is the spatial size of the output feature map.
\end{lemma}
\begin{proof}
We estimate the difference:
\begin{align}
	 |f[I](x)-f[J](x)| & = \left| r\left( \sum_y K(xd-y) J_1(y) \right)  -  r\left( \sum_y K(xd-y) J_2(y) \right) \right| \\
		 				 &\leq \left| \sum_y K(xd-y)( J_1(y) - J_2(y) ) \right| \\
		 				 &\leq \sum_y |K(xd-y)| | J_1(y) - J_2(y) | \\
		 				 &\leq C_I |s| \sum_y |K(xd-y)| \\
		 				 &\leq C_I |s| \sum_y \left| \sum_j w_j D_jG_{\sigma}(xd-y) \right| \\
		 				 &\leq C_I N \| W\|_{\infty} \| DG_{\sigma}\|_{\infty} |s|,
\end{align}
where $K= WDG_{\sigma}$, and we applied the Lipschitz continuity of $r$ in the first step, and the fact that $|J_1(x)-J_2(x)|\leq C_I |s|$.
\end{proof}

Using the previous lemma, we can now show that the output of a deep GaussNet is translation insensitive.
\begin{theorem}
	A deep sub-sampled GaussNet that is followed by a global average pooling layer, i.e.,
	\[
	F[I] = \frac 1 N \sum_x f_n\circ f_{n-1} \circ \cdots \circ f_1[I](x)
	\]
	is translation insensitive.
\end{theorem}
\begin{proof}
	We may now apply Lemma~\ref{lem:insensitive_input} successively to the network 
	\begin{equation}
	\hat f = f_n\circ f_{n-1} \circ \cdots \circ f_2
	\end{equation}
	to arrive at
	\begin{equation} \label{eq:multi_layer_insensitive}
		|\hat f[J_1](x) - \hat f[J_2](x)| \leq C_I  \| DG_{\sigma}\|_{\infty}^{n-1} \prod_{i=2}^{n}\| W_i\|_{\infty}N_i |s|,
	\end{equation}
	where $N_i$ is the spatial size in pixels of the feature map at the $i^{\text{th}}$ layer.
	
	We may set $J_1 = f_1[I]$ and $J_2 = f_1[T_sI]$, and apply \eqref{eq:multi_layer_insensitive}. By the $g$-translation covariance of a layer of a GaussNet, we have that $J_2 = T_{s/d}f_1[I]$. Therefore (using computations in Theorem~\ref{thrm:insense_sub_Gauss_layer}), 
	\begin{equation} \label{eq:trans_insensitivity_single_layer}
		|f_1[I](x) - f_1[T_sI](x)| \leq C_G N_1 \|I\|_{\infty} \|W_1\|_{\infty} |s|.
	\end{equation}
	Now combining \eqref{eq:multi_layer_insensitive} and \eqref{eq:trans_insensitivity_single_layer} with $C_I = C_G N_1 \|I\|_{\infty}$ in \eqref{eq:multi_layer_insensitive}, we have that
	\begin{equation}
		|(f_n\circ f_{n-1} \circ \cdots \circ f_1)[I](x) - (f_n\circ f_{n-1} \circ \cdots \circ f_1)[T_sI](x) | \leq 
		 C_G  \|I\|_{\infty}\| \| DG_{\sigma}\|_{\infty}^{n} \prod_{i=1}^{n}\| W_i\|_{\infty} N_i |s|.
	\end{equation}
	Therefore average the previous result over $x$, we have that 
	\begin{equation}
		|F[I](x)-F[T_sI](x)| \leq 
	 	C_G \|I\|_{\infty}\| DG_{\sigma}\|_{\infty}^{n} \prod_{i=1}^{n}\| W_i\|_{\infty} N_i |s|.
	\end{equation}
	
\end{proof}

\begin{remark}[Lack of Translation Invariance of Sub-Sampled GaussNets]
We show now that sub-sampled GaussNets are not in general translation invariant, thus justifying the need for the weaker translation insensitivity. Recall that
\begin{align}
\sum_x f[I](x) &= \sum_x r\left\{ \sum_{y,j} w_j D_j G_{\sigma}(xd-y) I(y) \right\} \\
\sum_x f[T_sI](x) &= \sum_x r\left\{ \sum_{y,j} w_j D_j G_{\sigma}(xd-y-s) I(y) \right\}.
\end{align}
For translation invariance to hold, the above two sums should be equal. One would like to
	attempt the change of variable $zd = xd-s$, which would result in $z=x-s/d$.  Note
	that the range of summation (of $x$ and $z$) would only be equal when $s$ is a
	multiple of $d$ so that the Gaussian in the two sums above would have the same
	arguments, and thus the summations would not be equal (for an arbitrary $I$)
	unless $s$ is a multiple of $d$. Hence, GaussNets with sub-sampling in  general will not be translation invariant.

\end{remark}

\begin{remark}[Choice of $\sigma$]
We note that in the proofs, the translation insensitivity arises from the fact that $G_{\sigma}(xd)-G_{\sigma}(xd-s)$ is small, which was bounded by a Lipschitz estimate. In practice, $\sigma$ should be large compared to $s_{\text{max}}d$ to be less sensitive (where $s_{\text{max}}$ is shift of maximum desired insensitivity) - in this case the difference between the Gaussian and shifted Gaussian is small. Note that the sub-sampling rate $d$ contracts the Gaussian, and thus the factor of $d$ in choosing the $\sigma$. So the higher the desired insensitivity to large shifts and the sub-sampling rate, the more smoothing one has to perform.
\end{remark}

\begin{remark}[Contractive Operator]
	We have thus far shown that sensitivity of the GaussNet layer to shift in the input 
	is proportional to the ``size" of the weights, the input, and the Gaussian
	derivatives. Thus, weight regularization and batch normalization are
	also both beneficial in terms of adding robustness to CNN. 
	This sensitivity can be driven to $0$ as more GaussNet layers are added if 
	each GaussNet layer is a contraction operator (i.e., the Lipschitz constant is less than 1) on the shift perturbations.
	For the standard 1d Gaussian density function, its $n$-order derivatives are bounded by $\kappa_n<1$ for $n<4$.  
	In theory,  to ensure that each Gauss layer is contractive, one only need to apply normlization techniques on $I$ to 
	ensure $\| I\|<1$ and learning with hard constraints to ensure $\|W\|<\frac{1}{N}$.
\end{remark}

\subsection{Existing CNNs are Not Translation Insensitive}
We show that translation insensitivity is not a property of existing CNNs with sub-sampling. We analyze a layer of a general sub-sampled CNN. Let $K$ represent the learned kernel in a layer of a CNN. Then a layer (for instance, VGG with an average pool) is given by
\begin{equation}
	f_{\text{old}}[I] = S_d r( K\ast I ), \mbox{ and } 
	F_{\text{old}}[I] = \frac 1 N \sum_x f_{\text{old}}(x).
\end{equation}
We show that this architecture is not translation insensitive:
\begin{align}
	F_{\text{old}}[ & T_sI](x) - F_{\text{old}}[I](x)  \\ 
	&=\frac 1 N \sum_x S_d[ r(K\ast I) - T_sr(K\ast I)](x) \\
	&= \frac 1 N \sum_x J(xd) - J(xd-s)
\end{align}
where $J = r(K\ast I)$. For any $d > 1$ this sum cannot be controlled by $|s|$ without any additional smoothness properties on the class of $J$, which we do not have on general images, nor feature maps. For instance, in the case that $d=2$ and $s=1$, we have that the difference above is $\sum_x J(2x)-\sum_x J(2x-1)$, and since the ranges of the functions $J(2x)$ and $J(2x-1)$ are in general distinct, the difference of the sums cannot be controlled.

Note that $K$, the kernel, which is learned, would not in general have smoothness guarantees.  Therefore, the difference $J(x)-J(x-s)$ cannot be Lipschitz bounded. In practice, CNNs are typically implemented with small supports (e.g., $3\times 3$ kernels) and thus shifted differences of the kernel could be large, breaking translation sensitivity. For translation insensitivity, the kernel would need to smoothly die down to zero at the boundary of its support.

From the argument, the lack of translation insensitivity is due to the lack of smoothness on the kernel $K$, rather than just sub-sampling; of course larger sub-sampling rates would require smoother learned kernels.


\subsection{Some Anti-Aliased CNNs are Insensitive}
One may ask whether blurring the output before sub-sampling in layers of the CNN to remove high frequency components and avoid aliasing effects would lead to a translation insensitive CNN. As we show now, provided the kernel is chosen appropriately to be smooth and to die down to zero, such as a Gaussian, the answer is in the affirmative. A layer with anti-aliasing is given by 
\begin{equation}
	f_{a}[I] = S_d\left\{ G_{\sigma} \ast r( K\ast I ) \right\}, \mbox{ and } F_{a}[I] = \frac 1 N \sum_x f_{a}[I](x)
\end{equation}
where $K$ is a typical pixelized kernel (e.g., $3\times 3$ as in ResNet). Similar to the computation in the previous sub-section, the anti-aliased CNN is not translation invariant.  It is, however, translation insensitive. Consider a single-layer anti-aliased CNN with average pooling:
\begin{align}
	& |  F_{a}[I] - F_a[T_sI]| \\
	&= \frac 1 N 
	\left|
		 \sum_{x,y} G_{\sigma}(xd-y) \left[ r ( K\ast I )(y) - 
		 r ( K\ast T_sI )(y) \right]
	\right| \label{eq:alias_1}\\
	&= \frac 1 N \left|
		\sum_{x,y} \left[ G_{\sigma}(xd-y)-G_{\sigma}(xd-y-s)  \right] r ( K\ast I )(y)
	\right| \label{eq:alias_2}\\
	&\leq C_G |s| \sum_y |K\ast I|(y) \label{eq:alias_3}\\
	&\leq N C_G \| K \|_{\infty} \| I \|_{\infty} |s|, \label{eq:alias_4}
\end{align}
where we have performed a change of variables in going from \eqref{eq:alias_1} to \eqref{eq:alias_2}, and used the Lipschitz continuity of the Gaussian and ReLu in going from \eqref{eq:alias_2} to \eqref{eq:alias_3}. 

Thus, as we see, one gets translation insensitivity, similar to our approach in the previous sub-section with a similar Lipschitz constant. As we can see, the only requirement on the anti-aliasing filter is that it be Lipschitz continuous. In practice, for finite data this would imply the need to die down to zero at the border of its support.


\section{Experiments}
\label{sec:experiments} 
\label{sec:expts}
We compare the robustness performance (to small translations) 
of our new architecture against ResNet.

{\bf Datasets}: We performed experiments on CIFAR-10 dataset for image classification.  
We also created a second dataset derived from CIFAR-10 
in which we down-sampled the original CIFAR-10 image to $30\times 30$ 
and then zero-pad the image by one pixel on each side to create a $32\times 32$ image. 
We call this second dataset
CIFAR10-ZP.

{\bf Shifted Test Sets}: We evaluate each network's insensitivity to shift perturbations
by testing them on shifted test sets. The shifted test sets are constructed by
shifting every image in test sets of CIFAR-10 and CIFAR-10-ZP at most 1 pixel in the $x$ and/or $y$ directions (8 different shifts). 
The missing values at borders after the shift is filled by copying the closest pixel in the 
image to the missing pixel on the border\footnote{On CIFAR-10-ZP, these shifts do not remove any content of the original image.}. 

{\bf Architectures}: Given a ResNet architecture ResNet-N, 
we replace its convolution layers that have 
kernel size of greater than $1$ with GaussNet layers.  
We keep the batch normalization, and the residual block structure.  
For every layer in ResNet-N in which there is sub-sampling, we add a corresponding 
sub-sampling layer~\footnote{implemented using an average pooling with stride of $2$} into the new architecture GaussNet-N.
For each GaussNet-N, we also create a version (no sub.) that is without any sub-sampling.  
We also have ResNet-N with Anti-aliasing.  This architecture is created
by applying a Gaussian filter to every sub-sampling layer within ResNet-N. 
Finally for GaussNet-50, we evaluated a \textit{Large} version that is created by replacing also 
the Conv1x1 layers within ResNet-50 with GaussNet convolutions layers. 
The architectures and their estimated sizes are listed in Table~\ref{tab:size}.  

\begin{table}[htp]
\caption{Model Size is proportional to \# of Parameters and the size of the Fwd/Back Pass} 
\begin{center}
\begin{tabular}{l|ccc} 
Arch/Size   &   \# of Params.   &   Fwd/Back Pass (MB)   &   Model Size (MB) \\ \hline
ResNet-18   &   11,173,962   &   11.25   &   53.89 \\
GaussNet-18 no sub.   &   7,515,466   &   120.5   &   149.18 \\
GaussNet-18   &   7,515,466   &   15.72   &   44.4 \\
ResNet-50   &   23,520,842   &   17.41   &   107.14 \\
GaussNet-50   &   19,751,690   &   22.64   &   98. \\
GaussNet-50 Large   &   66,567,370   &   22.64   &   276.59
\end{tabular}
\end{center}
\label{tab:size} 
\end{table}

{\bf Implementation Details}: To implement a GaussNet layer, we 
choose the supports of our Gaussians that are the sizes of feature maps.
These have to be chosen with wide support and should die down near the border of feature
map so as to avoid edge effects that would contribute to sensitivity. We use FFTs to
compute the convolution with large supports efficiently. The FFT of the Gaussian needs to
be only computed once.  To compute the derivatives of a Gaussian, we simply used
difference operators (Sobel-Feldman in our case) directly on the Gaussian filtered outputs.  
For experiments on small image data, the output convolution, due to its large support, is
inaccurate because of the lack of data near the borders. To mitigate the effects of this,
rather than performing average pooling, we first multiply the result just before the
pooling by a Gaussian centered at the center of the feature map and then compute the average.

{\bf Robustness Measures}: We measure both the probability of a change in classification in the shifted test sets, $\Delta_1 = 
\frac{1}{8|\mathcal T|}\sum_{I\in\mathcal T}\sum_{s} \mathbf{1}(c(I)\neq c(T_sI))$, where $\mathcal T$ is the test set, $|\mathcal T|$ is its size, and $c$ is the classifier output. We also measure the probability that at least one of the eight shifts of a test image results in a different classification, i.e., $\Delta_2 = \frac{1}{|\mathcal T|}\sum_{I\in \mathcal T} \mathbf{1}( \sum_s \mathbf{1}(c(I)\neq c(T_sI)) \geq 1 )$.

{\bf Training Details}: Using ADAM, we train each architecture including ResNet-N on both CIFAR-10 and CIFAR-10-ZP.  
We perform no training data augmentation in order to test the inherent invariances built in the architectures.
Without using data augmentations, the test accuracy of each architecture plateaus after a quick initial convergence to $80\%-82\%$.  
We did not push the trainings to beyond $100$ epochs to get a better test accuracy 
because our robustness measures are applicable at any level of test accuracy. 
The important thing is to select models that are comparable in terms of their test accuracies. 
For the robustness benchmark, we selected trained models of around or near $80\%$ test accuracy.  

{\bf Benchmark Results}:  The results are shown in Tables~\ref{tab:01}-~\ref{tab:04}. 
The GaussNets significantly outperformed the ResNets and their anti-aliased versions
in all possible combination of datasets and robustness measures.
There is a degradation of robustness across the board
going from CIFAR-10 to CIFAR-10-ZP ranging from only small degradations for the GaussNets to 
very large degradations for ResNets with anti-aliasing.

In training time, GaussNet-N architectures take 
longer than their ResNet counterparts mostly because of the FFT Gaussian filtering.  
However their training times are still comparable and are 
within the same order of magnitude.  For example, on a single GPU laptop, 
GaussNet-50 took about $16239$s to complete 50 epochs of training compared to $4859$s for ResNet-50. 
GaussNet-18 with no sub-sampling and GaussNet-50 Large in contrast take 1-2 orders of magnitude longer to train and only gain a slight advantage in robustness over their un-modified GaussNet-N counterparts.

\begin{table}[htp]
\caption{CIFAR-10 Shifted Test: Robustness(smaller is better).} 
\begin{center}
\begin{tabular}{l|ccc}
Arch/Robustness & $\Delta_1$ & $\Delta_2$ & Test-error  \\\hline
ResNet-18   &   $14.24\%$   &   $35.20\%$   &   $20.64\%$ \\
ResNet-18 + Anti-Alias   &   $10.86\%$   &   $25.69\%$   &   $19.82\%$ \\
GaussNet-18   &   $8.97\%$   &   $23.83\%$   &   $20.59\%$ \\
GaussNet-18 no sub.   &   $7.54\%$   &   $22.08\%$   &   $20.88\%$
\end{tabular} 
\end{center}
	\label{tab:01}
\end{table}

\begin{table}[htp]
\caption{CIFAR-10-ZP Shifted Test: Robustness(smaller is better).} 
\begin{center}
\begin{tabular}{l|ccc}
Arch/Robustness & $\Delta_1$ & $\Delta_2$ & Test-error  \\ \hline
ResNet-18   &   $19.99\%$   &   $47.94\%$   &   $20.63\%$ \\
ResNet-18 + Anti-Alias   &   $17.53\%$   &   $39.66\%$   &   $20.61\%$ \\
GaussNet-18   &   $9.44\%$   &   $24.94\%$   &   $21.03\%$ \\
GaussNet-18 no sub.   &   $8.96\%$   &   $25.13\%$   &   $20.81\%$
\end{tabular}
\end{center}
	\label{tab:02}
\end{table}

\begin{table}[htp]
\caption{CIFAR-10 Shifted Test: Robustness(smaller is better).} 
\begin{center}
\begin{tabular}{l|ccc} 
Arch/Robustness & $\Delta_1$ & $\Delta_2$ & Test-error  \\ \hline
ResNet-50   &   $12.31\%$   &   $31.38\%$   &   $19.06\%$ \\
ResNet-50 + Anti-Alias   &   $10.70\%$   &   $25.49\%$   &   $19.44\%$ \\
GaussNet-50   &   $8.70\%$   &   $23.50\%$   &   $19.68\%$ \\
GaussNet-50 Large   &   $8.43\%$   &   $22.93\%$   &   $19.93\%$
\end{tabular}
\end{center}
	\label{tab:03}
\end{table}

\begin{table}[htp]
\caption{CIFAR-10-ZP Shifted Test: Robustness(smaller is better).} 
\begin{center}
\begin{tabular}{l|ccc}
Arch/Robustness & $\Delta_1$ & $\Delta_2$ & Test-error  \\ \hline
ResNet-50   &   $15.60\%$   &   $39.61\%$   &   $20.70\%$ \\
ResNet-50 + Anti-Alias   &   $16.51\%$   &   $39.21\%$   &   $20.70\%$ \\
GaussNet-50   &   $11.18\%$   &   $27.96\%$   &   $20.98\%$ \\
GaussNet-50 Large   &   $10.41\%$   &   $27.24\%$   &   $20.86\%$
\end{tabular}
\end{center}
	\label{tab:04} 
\end{table}

\section{Conclusion} 
\label{sec:conclusion} 
We addressed the lack of translation invariance in modern CNNs by introducing a new CNN architecture, GaussNet that is translation invariant and a sub-sampled version that is translation insensitive. We showed analytically why existing CNNs are not only not translation invariant but also not translation insensitive.  We proved analytically that our new architecture is translation insensitive, owing to the enforced smoothness of the kernels. Empirically, we showed that GaussNets could be trained to achieve similar test accuracy as modern CNNs, while being much less sensitive to shifts.  This came at a reasonable increase in training and inference times.  We showed that the sub-sampled GaussNet was as insensitive as the non-sub-sampled one, allowing considerable gains in speed and less memory usage. Some aspects not explored in this paper were insensitivities to other transformations, e.g., scalings and deformations.  GaussNets are naturally adaptable to address these as well.  We plan to explore this in future work.

\appendix

\section{Additional Experimental Analysis}
\label{sec:add_expts}

\subsection{Plots of Translation Sensitivity Over Epochs}

In this section, we show additional evaluations of robustness/insensitivity of the trained models versus the epoch of training. 
Recall from the main paper that to evaluate the robustness/insensitivity: we first shift the 
test data of CIFAR-10 and CIFAR-10-ZP by the $8$ possible 1-pixel translations in the $x$ and $y$ directions
which results in a pair of shifted test sets; and then we evaluate the robustness of
the trained models on the respective shifted test sets via two measures of the probability of {\bf change in classification}.  
We measure $\Delta_1$, i.e. the probability of a change in classification in the shifted test sets. 
We also measure $\Delta_2$, i.e. the probability that at 
least one of the eight shifts of a test image results in a different classification. Note that lower $\Delta_i$ indicates greater insensitivity.

Figures~\ref{fig:11}-\ref{fig:41} show the plots of sensitivity. Note that the GaussNet-18 (and GaussNet-50) architectures are less sensitive than both their corresponding ResNet and ResNet-anti-aliased architectures, uniformly over all epochs.

\begin{figure}[H]
    \centering
\includegraphics[width=0.5\textwidth]{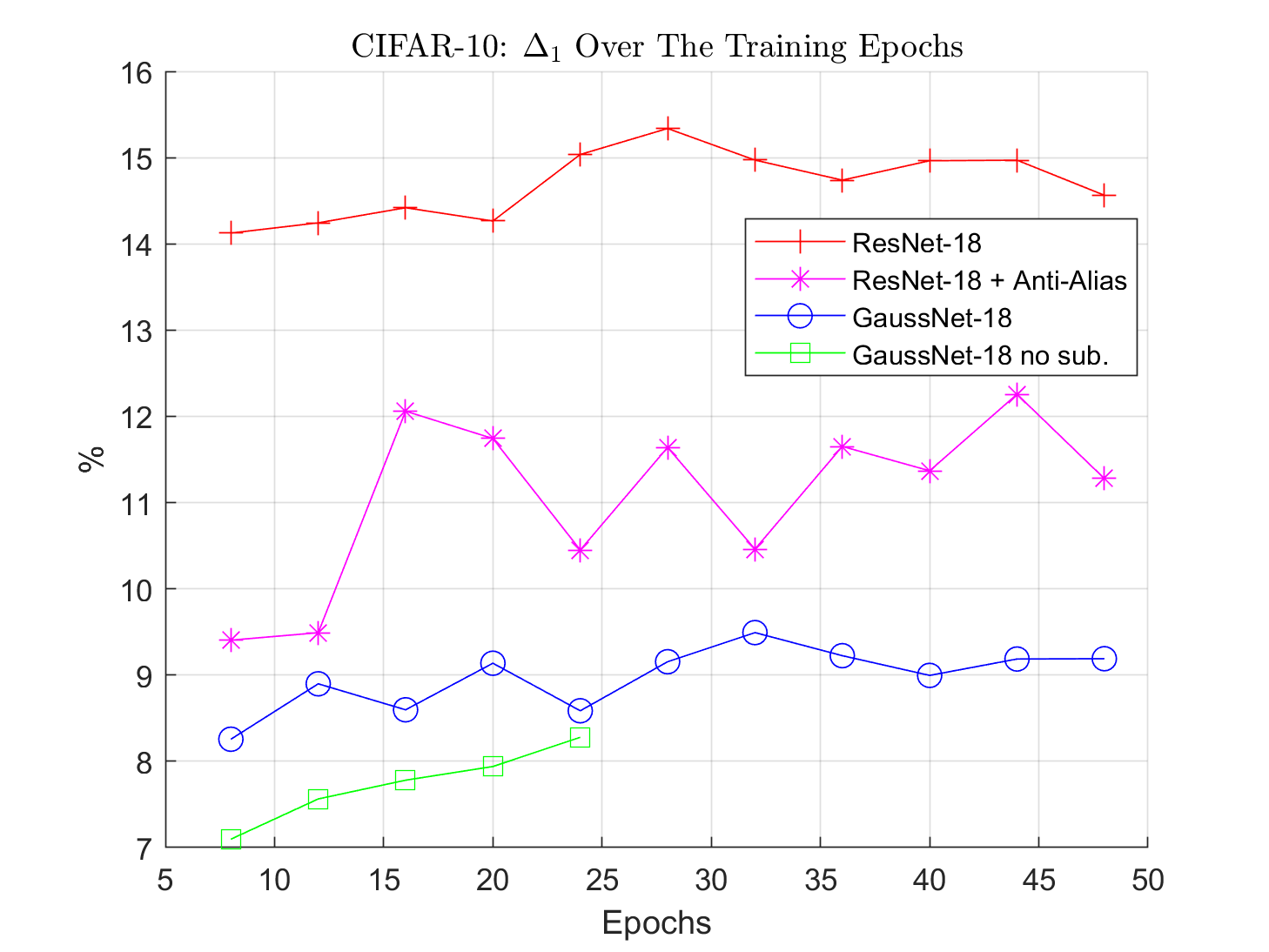}%
\includegraphics[width=0.5\textwidth]{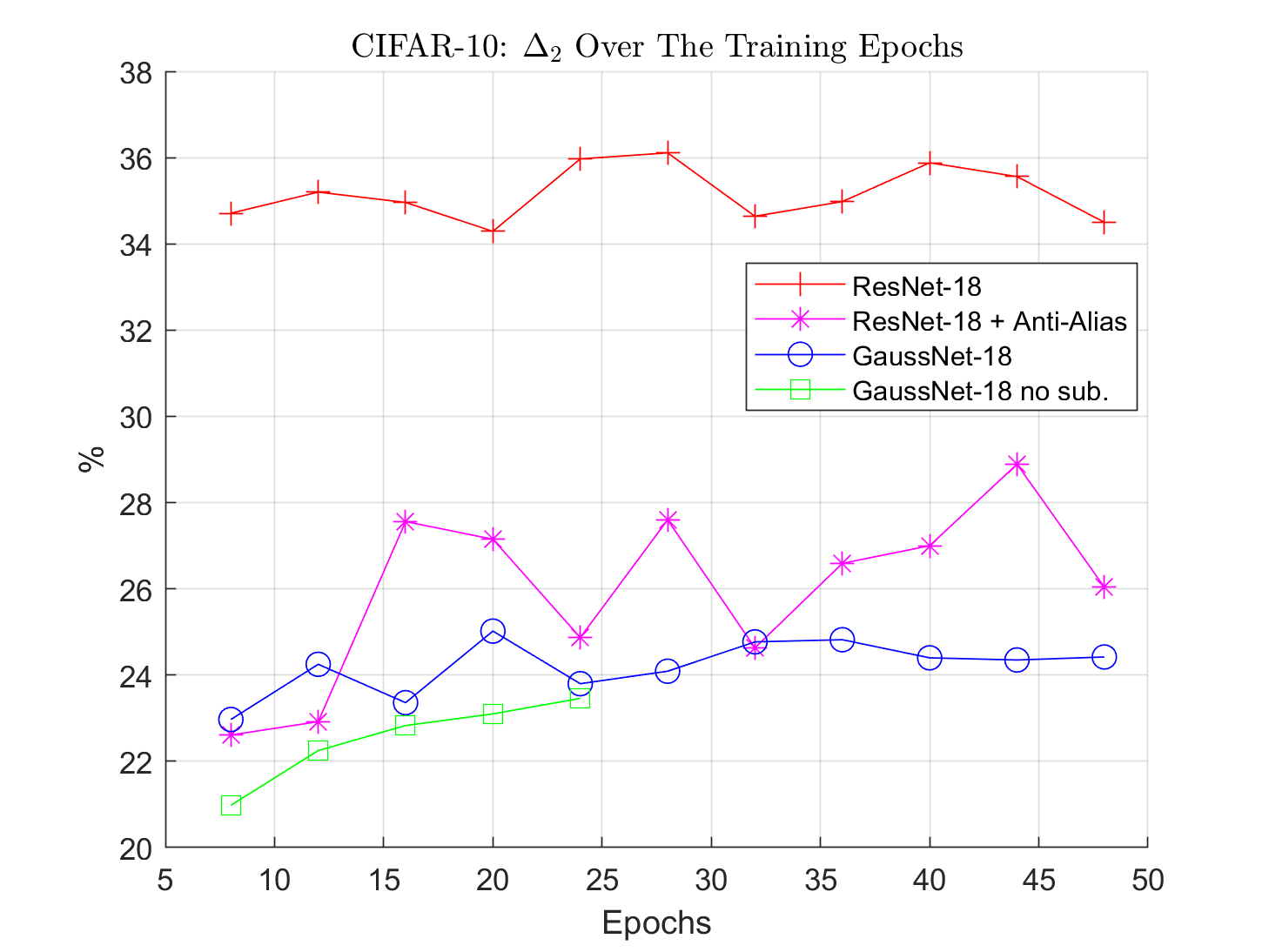}
\caption{CIFAR-10: Sensitivities $\Delta_1$ (left) and $\Delta_1$ (right) versus Epoch for *-18 Architectures. Lower is better.}
\label{fig:11}
\end{figure}

\begin{figure}[H]
    \centering
\includegraphics[width=0.5\textwidth]{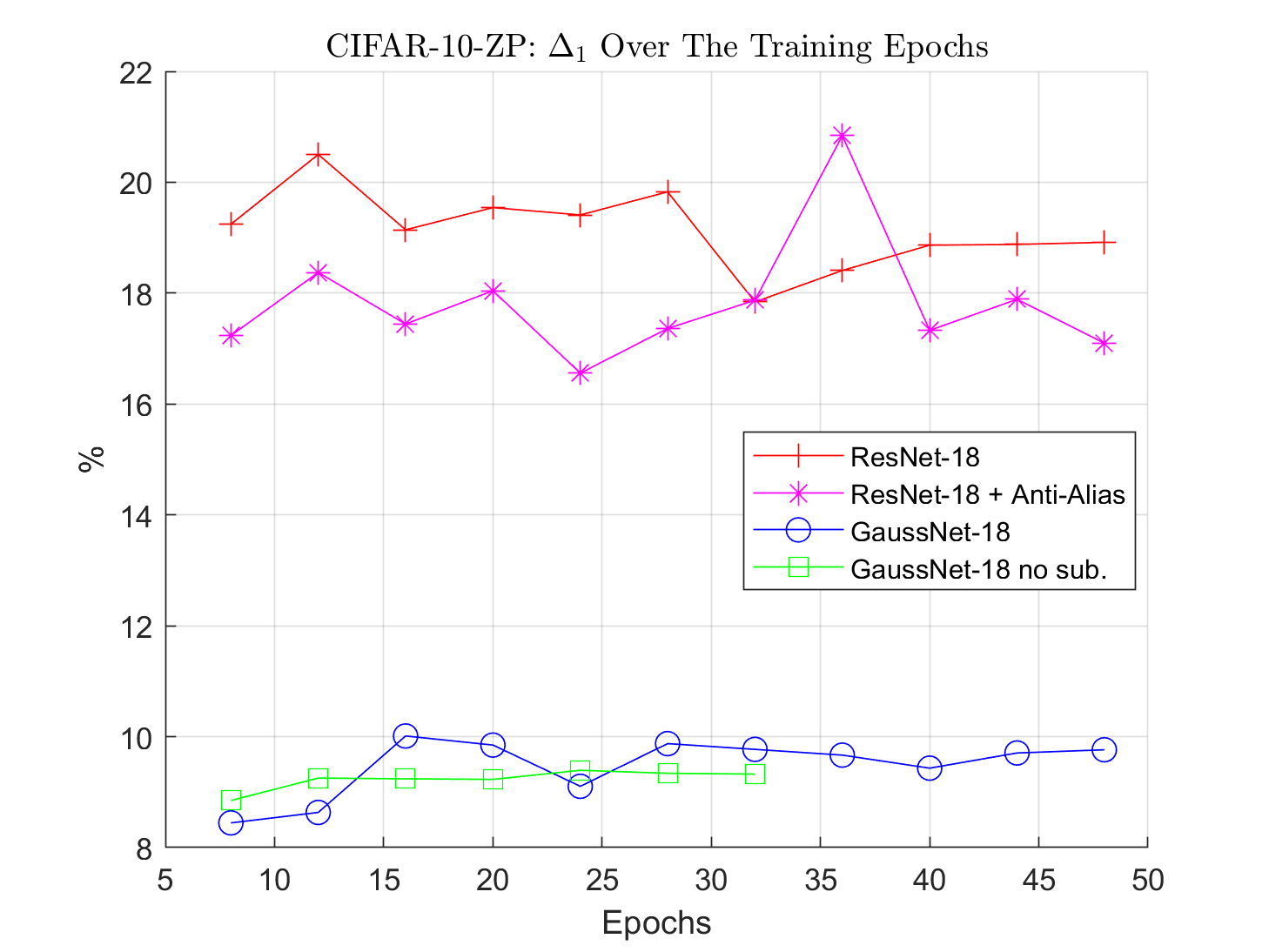}%
\includegraphics[width=0.5\textwidth]{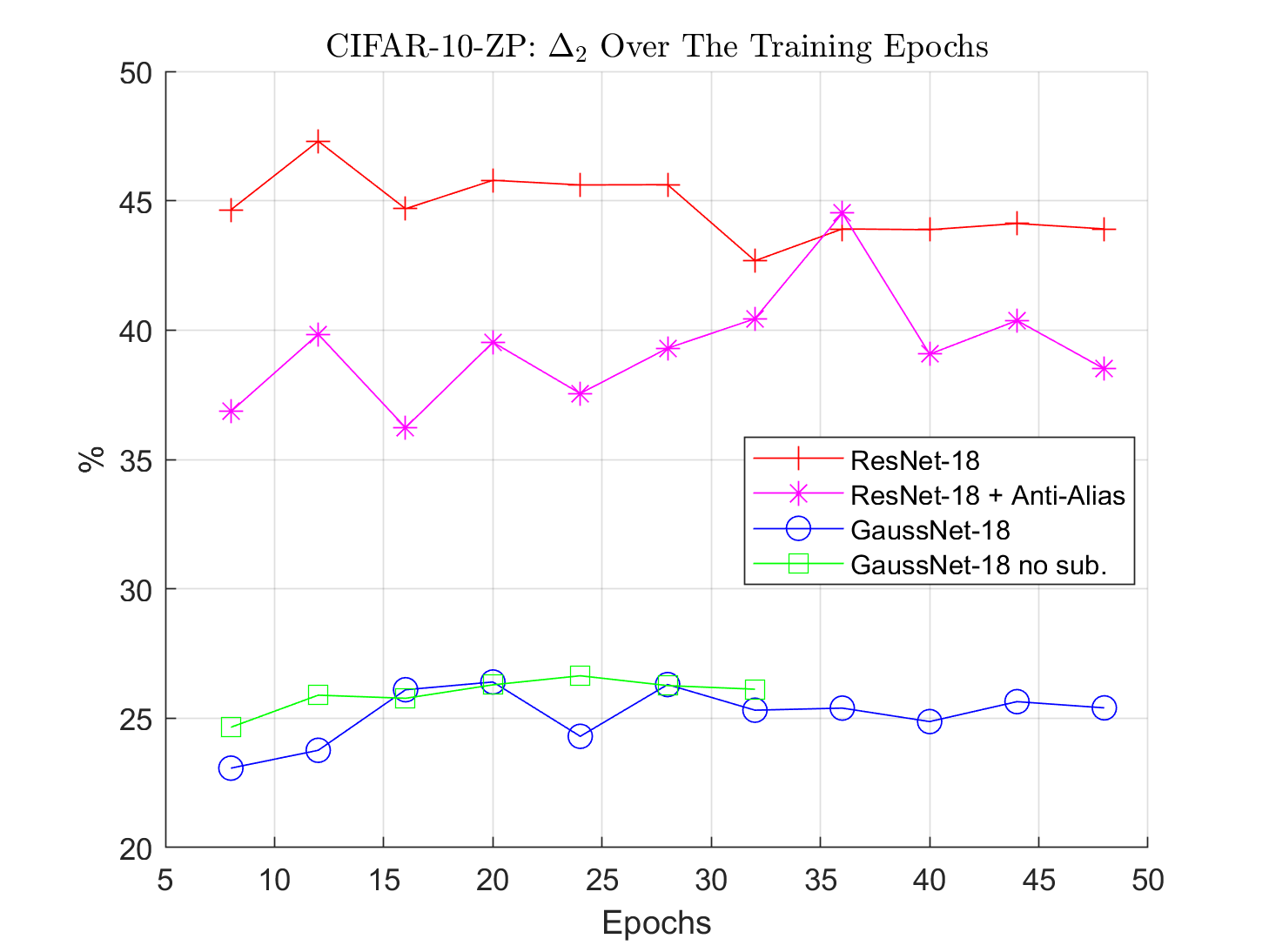}
\caption{CIFAR-10-ZP: Sensitivities $\Delta_1$ (left) and $\Delta_2$ (right) versus Epoch for *-18 Architectures. Lower is better.}
\label{fig:21}
\end{figure}

\begin{figure}[H]
    \centering
\includegraphics[width=0.5\textwidth]{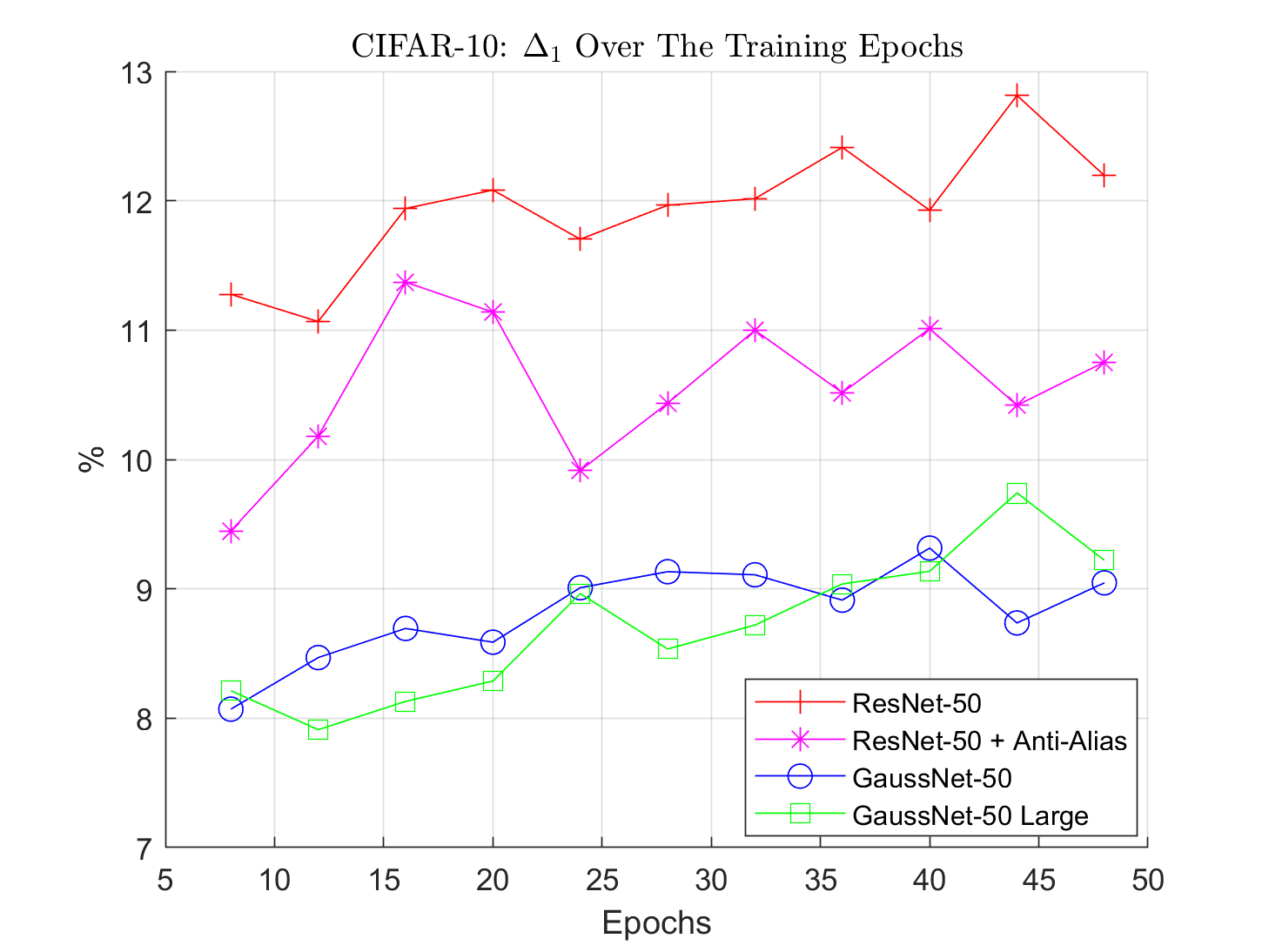}%
\includegraphics[width=0.5\textwidth]{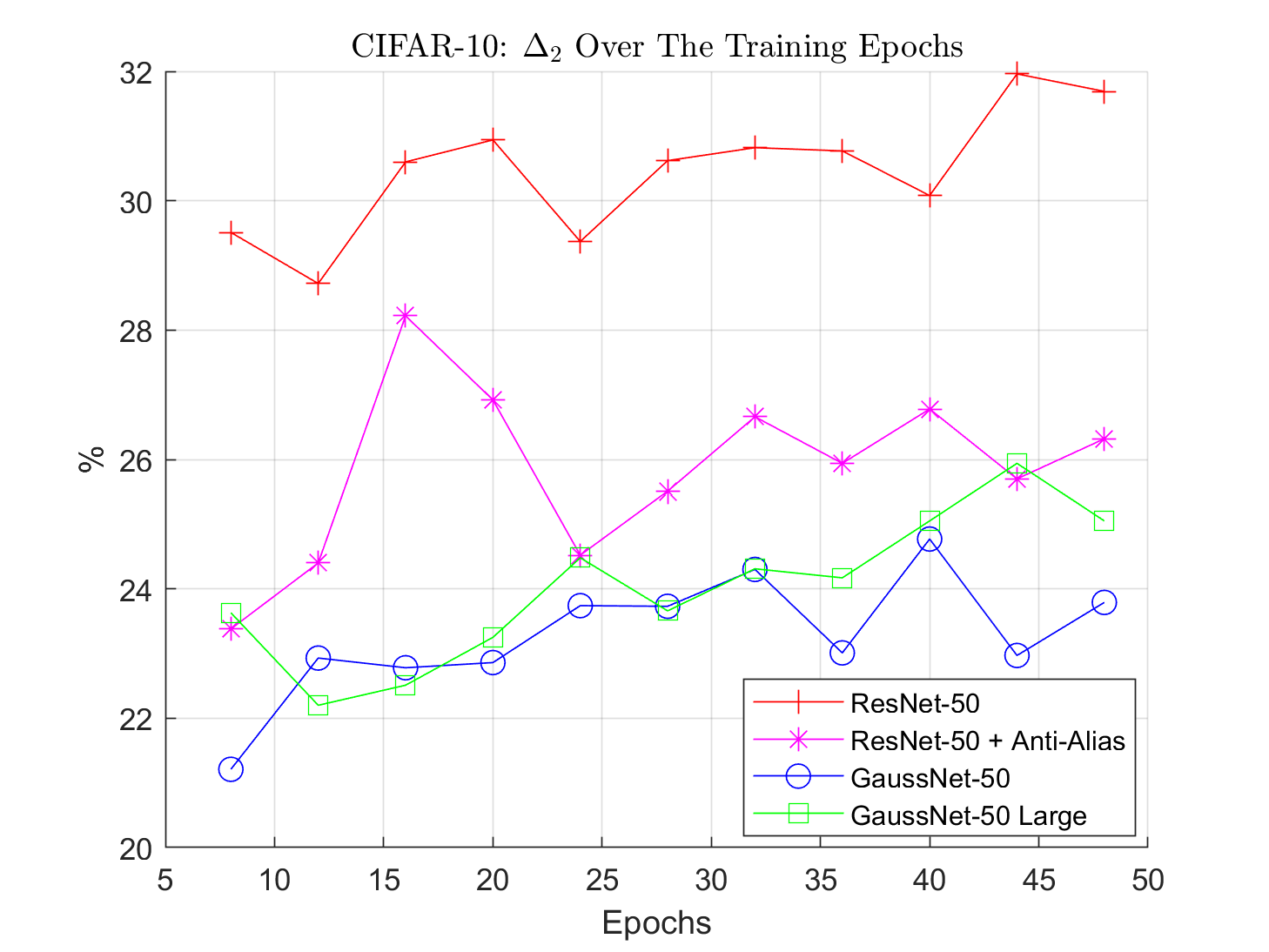}
\caption{CIFAR-10: Sensitivities $\Delta_1$ (left) and $\Delta_2$ (right) versus Epoch for *-50 Architectures. Lower is better.}
\label{fig:31}
\end{figure}

\begin{figure}[H]
    \centering
\includegraphics[width=0.5\textwidth]{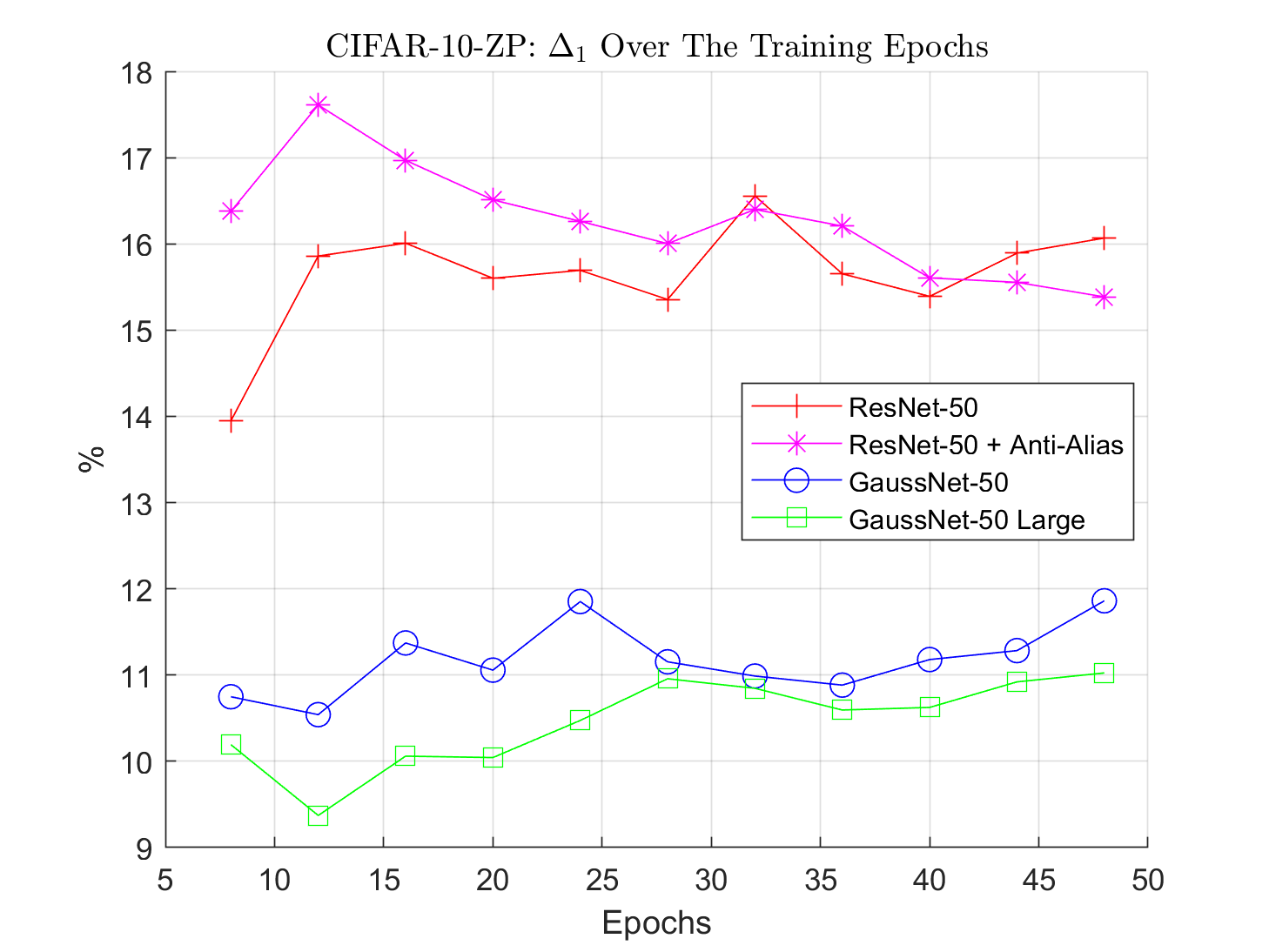}%
\includegraphics[width=0.5\textwidth]{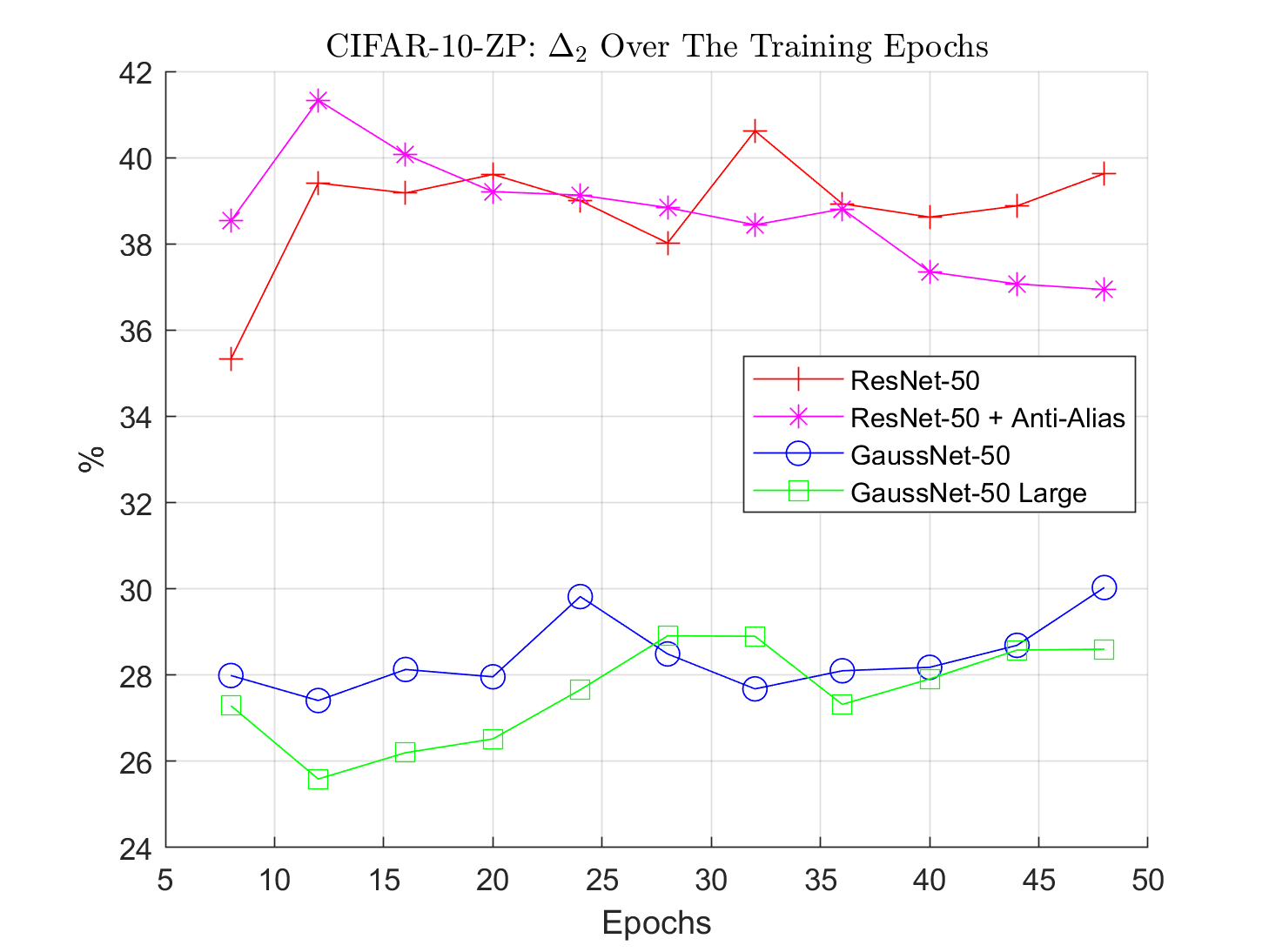}
\caption{CIFAR-10-ZP: Sensitivities $\Delta_1$ (left) and $\Delta_2$ (right) versus Epoch for *-50 Architectures. Lower is better.}
\label{fig:41}
\end{figure}

\subsection{Plots of Test Errors Over Epochs}

In this section, we analyze the test error (on the original test set) over epochs to show that our insensitivity does not come at a price of appreciable test-error loss. We include the test error performance versus epoch of training for  all the combinations of models and test sets from CIFAR-10 and CIFAR-10-ZP. This is shown in Figures~\ref{fig:test1}-\ref{fig:31}. Notice in Figure~\ref{fig:test1}, the architectures based on ResNet-18 all perform around the 20\% error mark, and the ResNet-18-anti-alias seems to give the best test error, though it seems to fluctuate the most over epochs. ResNet-18 is next best in terms of test error, though not much different than the GaussNet-18 architectures. Figure~\ref{fig:31} shows the plot of the test-error for the architectures based on ResNet-50. Notice that all the architectures perform similarly and the differences between the architectures are even less pronounced than the ResNet-18-based architectures.

\begin{figure}[H]
    \centering
\includegraphics[width=0.5\textwidth]{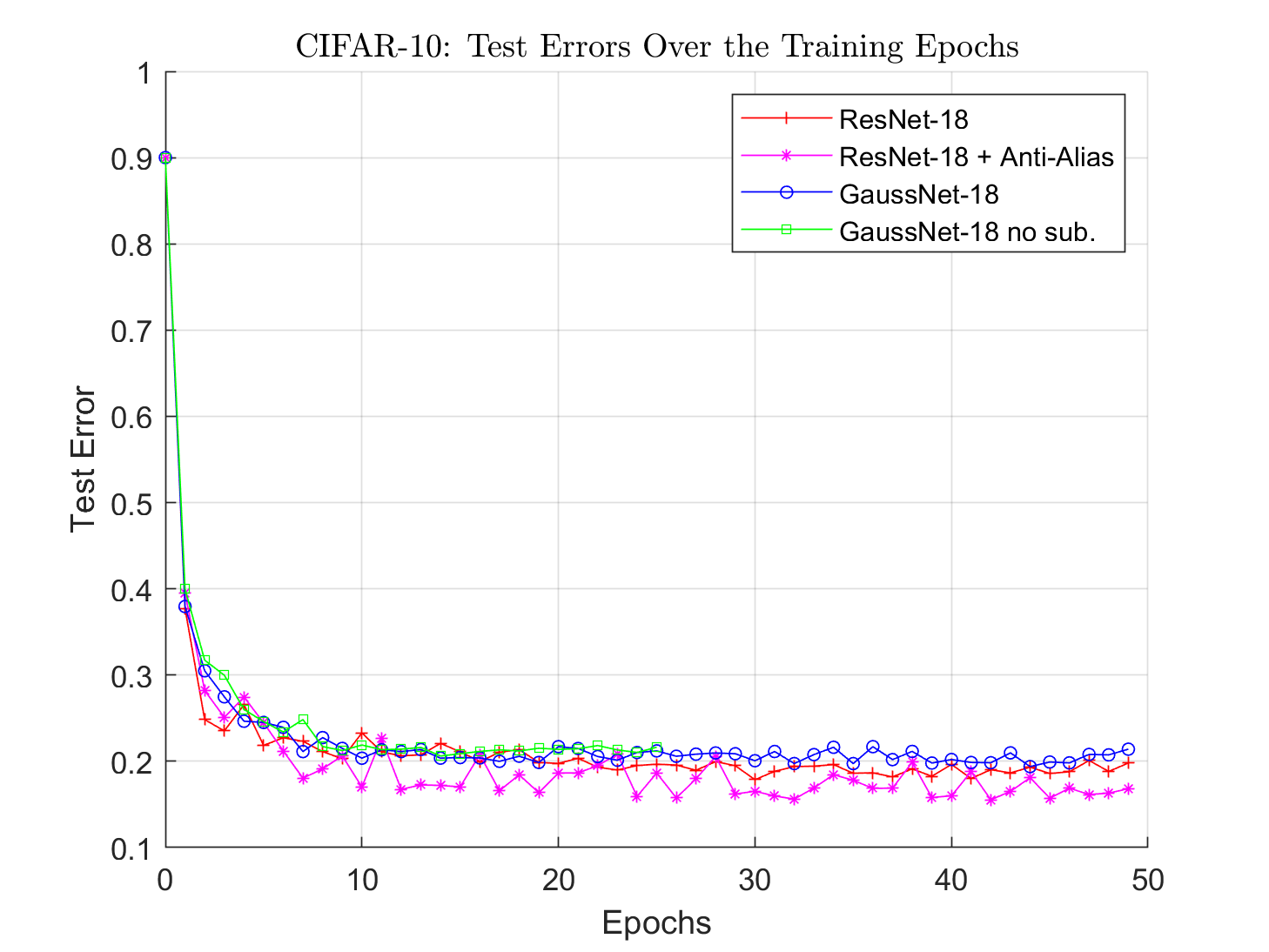}%
\includegraphics[width=0.5\textwidth]{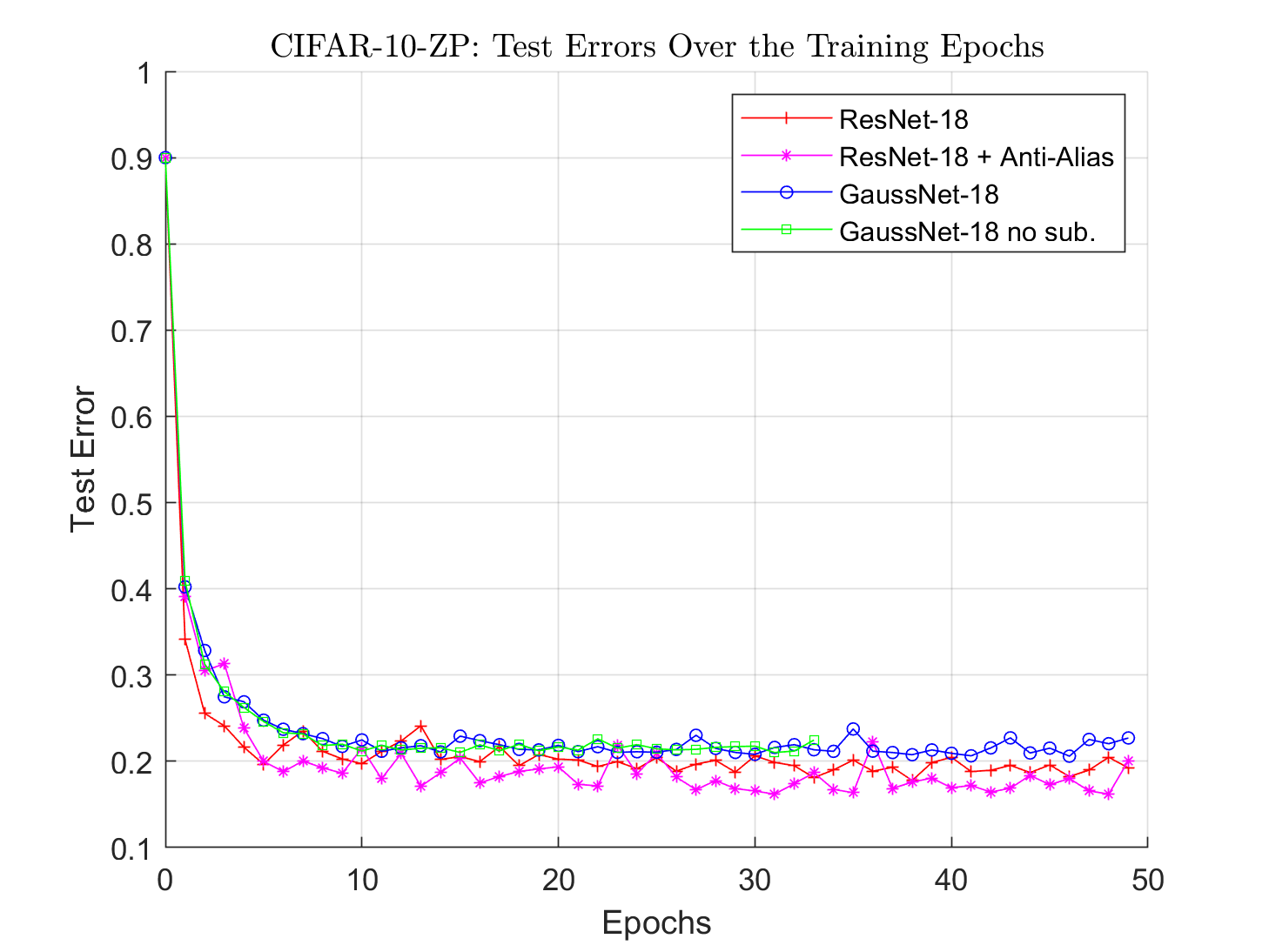}
\caption{Test Errors for *-18 Architectures: CIFAR-10 (left) and CIFAR-10-ZP (right).}
\label{fig:test1}
\end{figure}


\begin{figure}[H]
    \centering
\includegraphics[width=0.5\textwidth]{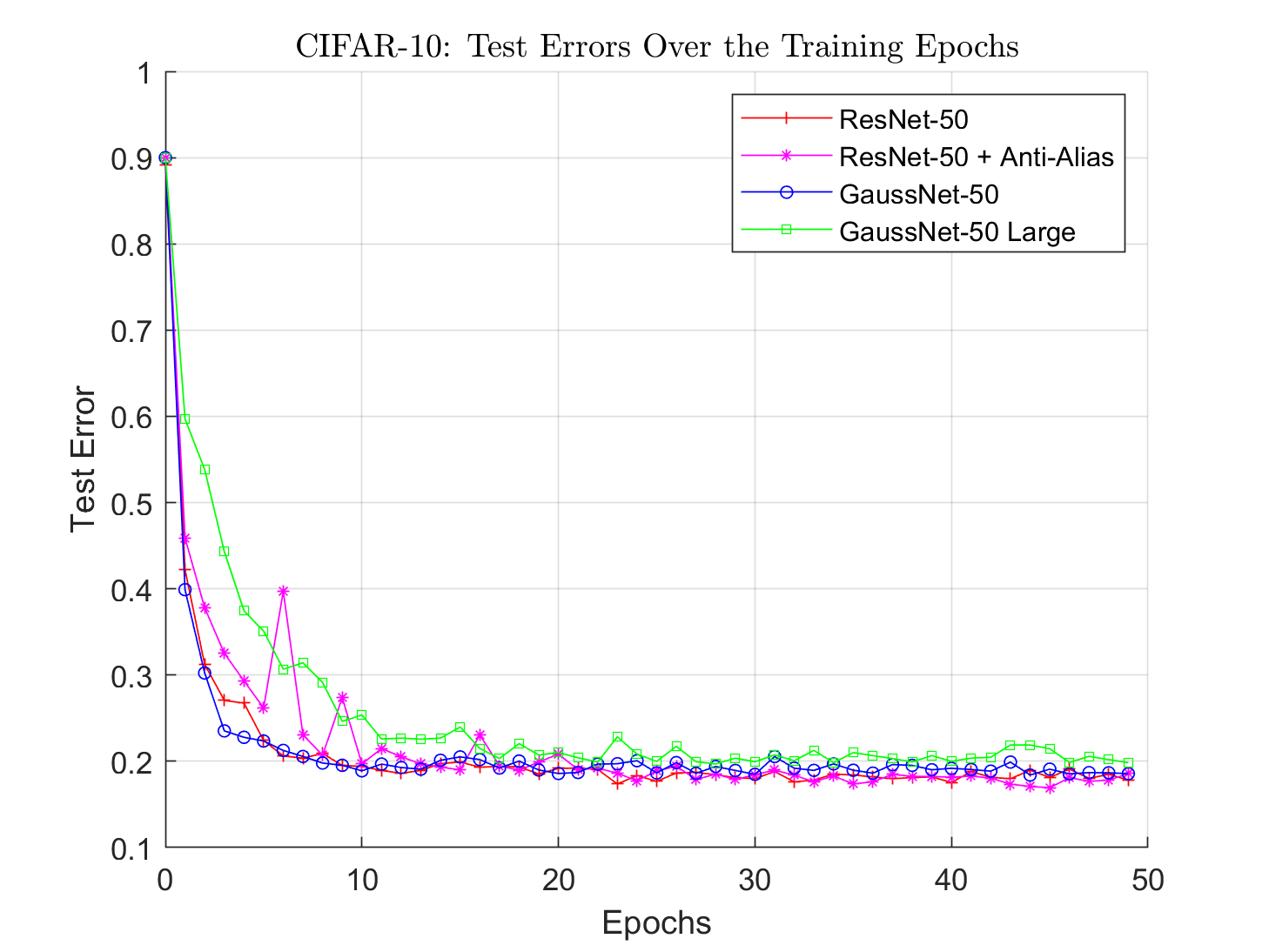}%
\includegraphics[width=0.5\textwidth]{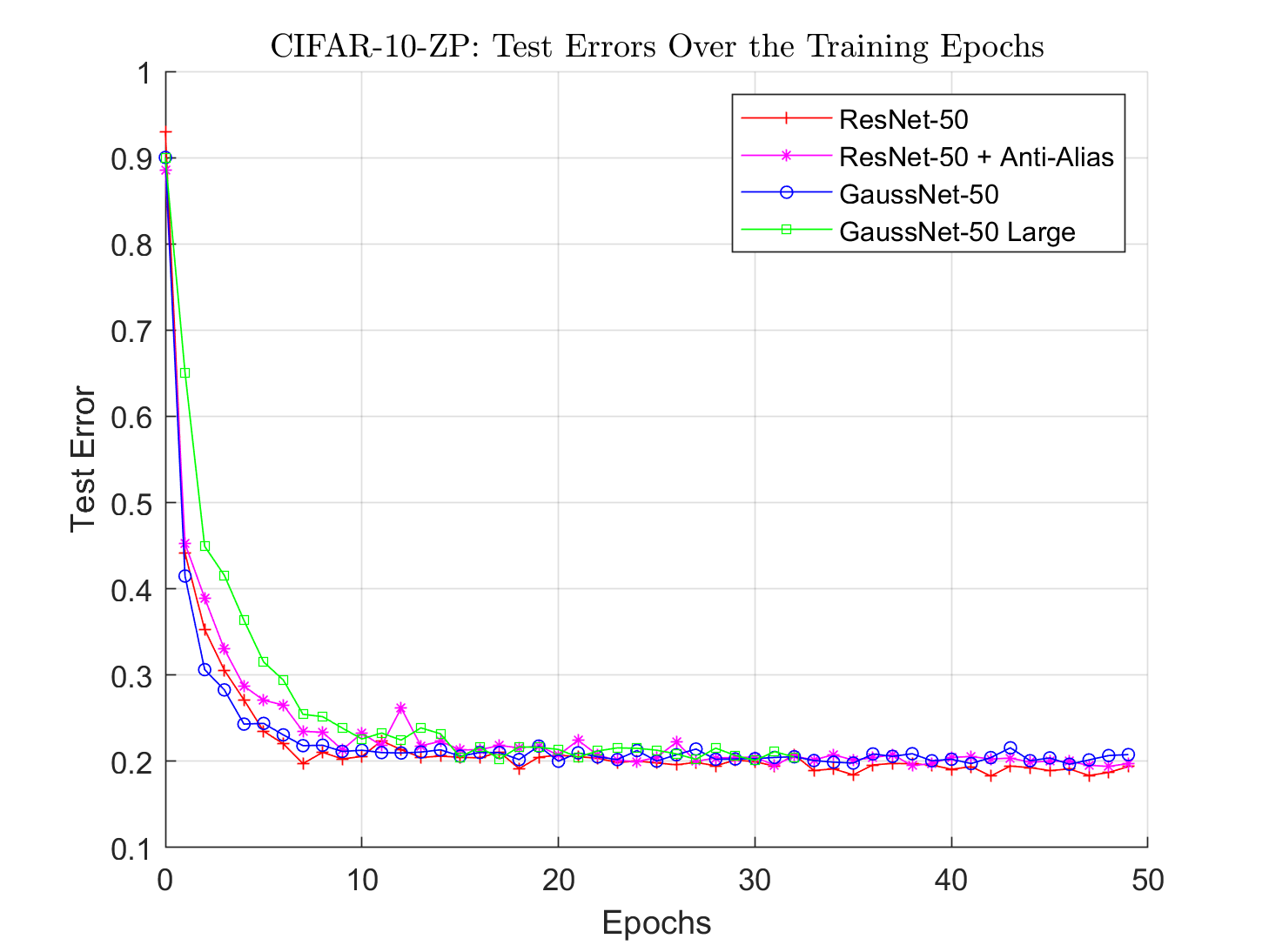}
\caption{Test Errors for *-50 Architectures: CIFAR-10 (left) and CIFAR-10-ZP (right).}
\label{fig:31}
\end{figure}


\subsection{Analysis of Sensitivity as $\sigma$ Varies}

In all of the experiments and the previous plots, we chose $\sigma=0.763$. In practice, this is a hyper-parameter to optimize. Here, we analyze sensitivity as $\sigma$ varies. We chose $\sigma \in \{ 0.3, 0.76, 1.3, 2.3\}$. The plots of sensitivity for GaussNet-50 are shown in Figure~\ref{fig:s12}, which shows the sensitivity for 5 sample epochs. This shows that the sensitivity does vary, but not by much.  Note that for each such $\sigma$, the GaussNet-50 is still more translation insensitive than ResNet-50. For $\sigma$ too small and too large, the sensitivity is large, and thus there seems to be an optimal value in the middle.

\begin{figure}[H]
    \centering
\includegraphics[width=0.5\textwidth]{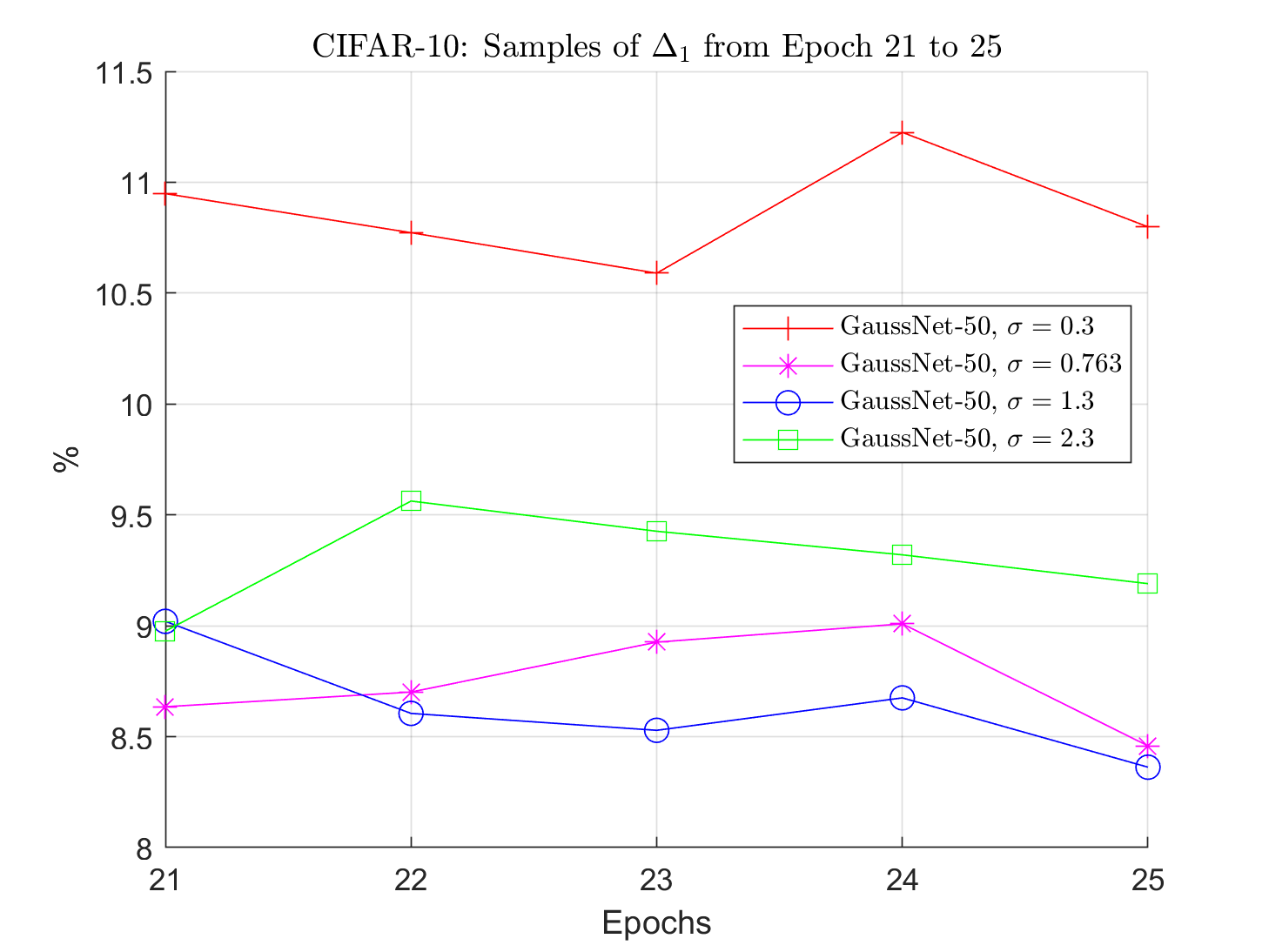}%
\includegraphics[width=0.5\textwidth]{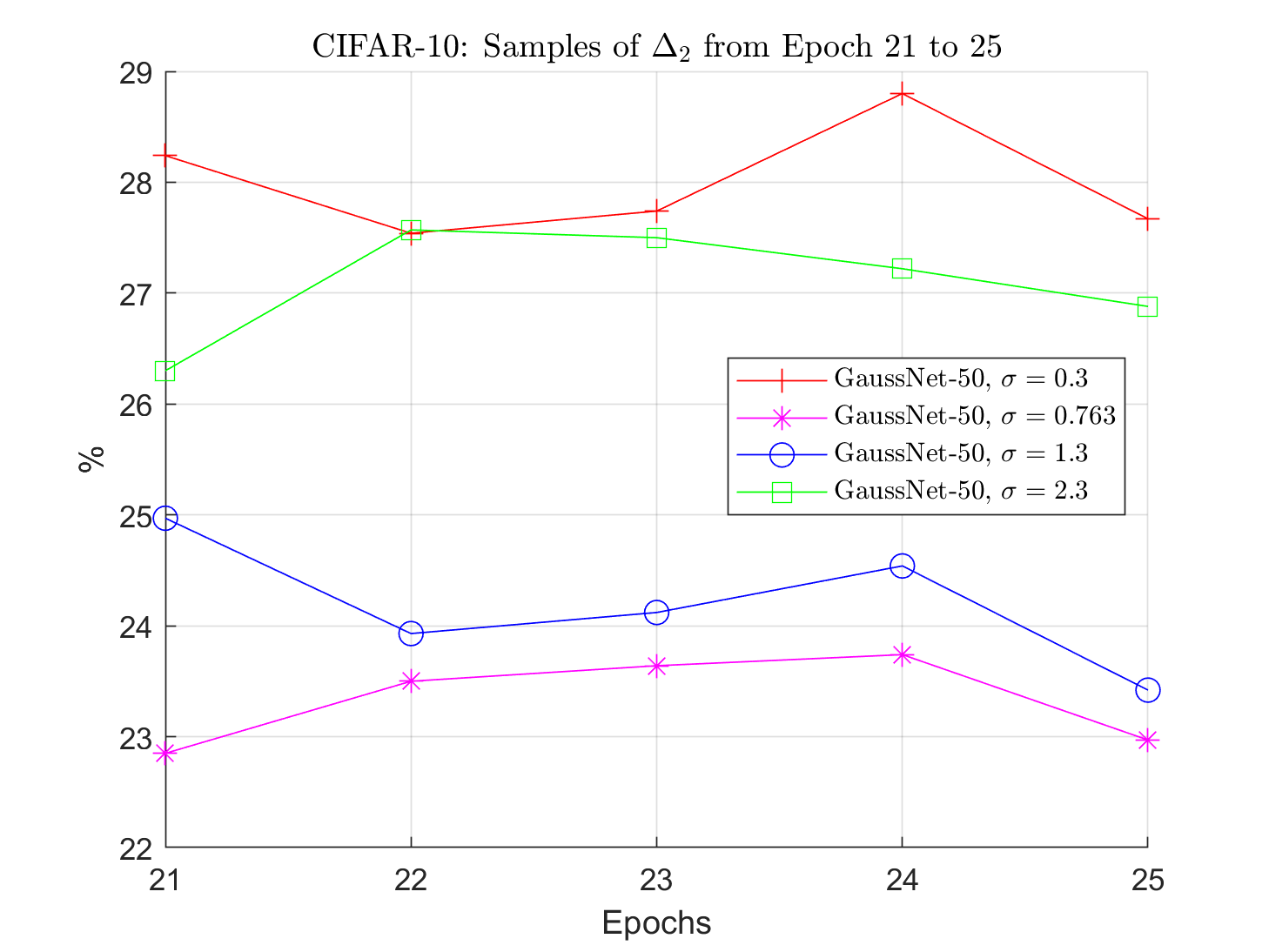}
\caption{CIFAR-10: Sensitivities $\Delta_1$ (left) and $\Delta_2$ (right) versus Epoch for GaussNet-50 for different $\sigma$. Lower is better.}
\label{fig:s12}
\end{figure}

\subsection{Analysis of Inference Times}

In this section, we show the inference times for each of the architectures evaluated in this study.  These times have been recorded on a single NVIDIA GeForce 2080-RTX Max-Q GPU. Tables~\ref{tab:infer01} and~\ref{tab:infer02} shows the absolute
inference times measured in seconds for the entire CIFAR-10 test set as well as the normalized times. 
While it may look surprising that GaussNet-50 is faster than GaussNet-18 in 
inference (and training as well), note that in GaussNet-50, the Conv1x1 layers of ResNet-50 are
not replaced since they are already shift covariant. ResNet-18 in contrast does not have 1x1 convolutions, and thus all its convolution layers are replaced with the Gauss-Hermite approximation.  Despite having a deeper architecture, the actual number
of GaussNet layers in the GaussNet-50 is the same as in GaussNet-18. In addition, on average 
the GaussNet layers within GaussNet-50 are deeper; these deep layers have smaller feature maps due to the earlier sub-samplings, and thus are faster to process. Finally, there is one extra sub-sampling in GaussNet-50 in comparison to GaussNet-18 thus further reducing the average size of input feature maps to the GaussNet layers. 

Thus, the translation insensitivity comes at a moderate price in inference time (GaussNet-50 is slower by a factor of 7.34x) than its corresponding ResNet-50 architecture, and roughly 3x in training time. However, we have not explored optimizing the GaussNet layer computations (e.g., the Gaussian filtering and derivatives) in terms of model parallelization and memory usage, which could lead to gains in speed. Furthermore, we simply did a direct replacement of convolutions in ResNet-50, and we have not explored the most optimized GaussNet architecture, so there are potentially speed-ups with reduced number of layers.

\begin{table}[htp]
\centering
\caption{Inference Times on CIFAR-10 Test Set: *-18 Architectures} 
\begin{center}
\begin{tabular}{l|cc} 
Arch/Stats   &   Absolute (s)   &   Normalized \\ \hline
ResNet-18   &   2.64   &   1.00x \\
ResNet-18 + Anti-Alias   &   3.70   &   1.40x \\
GaussNet-18   &   60.16   &   22.77x \\
GaussNet-18 no sub.   &   414.45   &   156.89x
\end{tabular}
\end{center}
\label{tab:infer01} 
\end{table}

\begin{table}[htp]
\centering
\caption{Inference Times on CIFAR-10 Test Set: *-50 Architectures} 
\begin{center}
\begin{tabular}{l|cc} 
Arch/Stats   &   Absolute (s)   &   Normalized \\ \hline
ResNet-50   &   3.47   &   1.00x \\
ResNet-50 + Anti-Alias   &   4.11   &   1.19x \\
GaussNet-50   &   25.47   &   7.34x \\
GaussNet-50 Large   &   160.79   &   46.33x
\end{tabular} 
\end{center}
\label{tab:infer02} 
\end{table}


{\small
\bibliographystyle{ieee_fullname}
\bibliography{refs/tim}
}


\end{document}